\documentclass[11pt]{article}

\usepackage{arxiv}
\usepackage[utf8]{inputenc}
\usepackage{times}
\usepackage{url}
\usepackage{booktabs}
\usepackage{amsmath,amssymb,amsfonts,amsthm}
\usepackage{nicefrac}
\usepackage{microtype}
\usepackage{xcolor}
\usepackage{graphicx}
\usepackage{multirow}
\usepackage{array}
\usepackage[numbers,sort&compress]{natbib}
\usepackage{algorithm}
\usepackage{algpseudocode}
\usepackage{enumitem}
\usepackage{tikz}
\usetikzlibrary{positioning,arrows.meta,calc}
\usepackage[hidelinks]{hyperref}
\usepackage{doi}
\usepackage[capitalise,nameinlink]{cleveref}
\usepackage{caption}
\usepackage{float}
\setlist{nosep,leftmargin=*}

\definecolor{yel}{HTML}{F7D774}\definecolor{yelD}{HTML}{B8860B}
\definecolor{pur}{HTML}{C9A7E0}\definecolor{purD}{HTML}{6E3B9E}
\definecolor{blu}{HTML}{8DB6EC}\definecolor{bluD}{HTML}{2A5DA8}
\definecolor{grn}{HTML}{9CCF8F}\definecolor{grnD}{HTML}{3D7A2F}
\definecolor{ora}{HTML}{F6C38F}\definecolor{oraD}{HTML}{C0571B}
\definecolor{tea}{HTML}{8ED1C6}\definecolor{teaD}{HTML}{217A6E}

\newtheorem{proposition}{Proposition}

\newcommand{\sj}{\textsc{Shift-Jev}}
\newcommand{\jev}{LLM-as-Jev}
\newcommand{\tshift}{\tau_{\mathrm{shift}}}

\renewcommand{\shorttitle}{\sj: Shift-Aware Abstention and Adaptation}

\title{\sj: Certified Abstention and Gain-Gated Continual Adaptation for LLM-Based Decision Models under Distribution Shift}

\author{%
  Muhammad Atif Saeed \\
  {\small Department of AI \& Data Science} \\
  {\small National University of Computer \& Emerging Sciences} \\
  {\small Islamabad, Pakistan} \\
  {\small\texttt{atif.saeed@isb.nu.edu.pk}}
  \And
  Azka Atiq \\
  {\small Department of AI \& Data Science} \\
  {\small National University of Computer \& Emerging Sciences} \\
  {\small Islamabad, Pakistan} \\
  {\small\texttt{azka.atiq@isb.nu.edu.pk}}
  \AND
  \rule{0pt}{16pt}Adarsh Kumar \\
  {\small Data Engineer} \\
  {\small University of Arkansas at Little Rock} \\
  {\small Arkansas, USA} \\
  {\small\texttt{akumar1@ualr.edu}}
  \And
  \rule{0pt}{16pt}Akhtar Jamil \\
  {\small Department of Computer Science} \\
  {\small National University of Computer \& Emerging Sciences} \\
  {\small Islamabad, Pakistan} \\
  {\small\texttt{akhtar.jamil@nu.edu.pk}}
  \And
  \rule{0pt}{16pt}Uzair Iqbal \\
  {\small Department of Software Engineering} \\
  {\small Faculty of Computer Science \& Information Technology} \\
  {\small University of Malaya, Malaysia} \\
  {\small\texttt{uzairiqbal@um.edu.my}}
  \And
  \rule{0pt}{16pt}Roland Aydin \\
  {\small German Research Center for } \\
  {\small Artificial Intelligence (DFKI)} \\
  {\small Saarbrücken, Germany} \\
  {\small\texttt{roland.aydin@uni-saarland.de}}
}
\date{October 2026}

\begin{document}
\maketitle

\begin{abstract}
Large language models (LLMs) can act as decision models that return a probability distribution over a fixed set of options, but deployed decision models face domain, label-prior, emerging-class and candidate-set shift. We study how such models fail under shift and propose \sj{}, a monitoring, abstention and continual-adaptation procedure for them. It scores the risk of each decision from uncertainty, format and option-consistency signals, certifies abstention thresholds with Learn-then-Test, monitors the decision stream with drift and label-shift detectors, and accepts a model update only when an expected-gain gate finds a significant held-out improvement without loss on earlier regimes. On a benchmark of ten shifted test environments with Qwen3-0.6B and Qwen3-4B (three seeds each) and Phi-4-mini, fine-tuning with KL anchors and temperature scaling lowers calibration error under shift from 0.289 to 0.087. At a target risk of 10\%, a learned uncertainty score with a certified threshold exceeds the target in 3--5 of 27 environment--seed tests, against 5--15 for a max-probability score with the same certification and 9--14 for conformal prediction sets; drift-based routing changes little. A predicted-label $\chi^2$ test detects 95--99\% of class-prior shifts that confidence- and embedding-based detectors miss. In continual adaptation \sj{} reaches 70.5\% accuracy on Qwen3-0.6B, above replay, LwF, EWC and fine-tuning baselines in every seed, whereas on Qwen3-4B its conservative gate leaves it 0.4--1.4 points behind them.
\end{abstract}


\section{Introduction}
\label{sec:intro}

Many software systems choose among a small set of predefined options: they route a support ticket, identify the intent of a customer message, check a document against a rule or grade a model response. General chat LLMs can perform these tasks, but their free-form answers must be parsed, they are slow to generate and they carry no usable confidence. \emph{Decision models} such as Jev~\citep{almeida2026jev} instead return a probability distribution over the candidate options in a single call, which downstream code consumes directly. Such models have recently been used as cheap first-stage judges~\citep{li2026jevjudge,rao2026jevrubric}, as detectors of alignment failures~\citep{guo2026justaskjev}, for social-science annotation~\citep{denney2026jevpolsci} and for coding administrative records~\citep{rafe2026crash}. \citet{li2026llmasjev} showed that an ordinary instruction-tuned LLM already behaves as a decision model (\jev): listing the options with bracketed numbers, pre-filling ``\texttt{Best answer: [}'' and scoring each option identifier through a prefix-free trie yields calibrated decisions without training, and a tree-factorized loss with KL anchors fine-tunes the model without damaging its chat behaviour.

These evaluations use test data that resembles the calibration data. Deployment does not. A routing model meets new products and phrasing (\emph{domain shift}), changing frequencies of the correct options (\emph{class-prior shift}), requests that no offered option covers (\emph{emerging classes}), and option lists that are reordered, reworded, extended or missing the right answer (\emph{candidate-set shift}). Because software acts on the returned probabilities automatically, a model that is confidently wrong under shift is more harmful than one that fails visibly. LLM judges are over-confident~\citep{tian2025overconfidence}, conformal guarantees for LLMs fail under dataset shift unless re-verified~\citep{kotte2026crccertify,lin2026dscp}, risk controlled over a population can exceed its budget within groups~\citep{salem2026hgcrc,rahman2026rolecrc}, and sequential fine-tuning can collapse coverage before it lowers accuracy~\citep{shihab2026continualcal}. No existing work studies how LLM decision models behave under these shifts or how to keep their errors within a budget.

We address three research questions. \textbf{RQ1}: how do the accuracy and calibration of frozen and fine-tuned LLM decision models degrade under the four shift types? \textbf{RQ2}: can abstention keep the realized error rate of answered inputs below a target $\alpha$ under shift, and which signals and threshold rules achieve this? \textbf{RQ3}: when labelled data from a new regime arrive, can the model adapt without forgetting earlier regimes while declining updates that do not help?

We answer them with \sj{} (\cref{fig:framework}), which adds three modules to the \jev{} interface. Module~A monitors each decision and the decision stream. Module~B abstains when a learned risk score exceeds a threshold certified by Learn-then-Test; drift distances are used to route inputs to a threshold but never enter the score. Module~C adapts the LoRA adapter to a new regime and keeps the update only when it measurably helps.

This paper makes three contributions. First, it introduces a shift benchmark for LLM decision models with one in-distribution and ten shifted test environments covering all four shift types, seven external benchmarks at the base paper's sample sizes, stream protocols at three severities and a four-stage continual protocol that includes a no-shift control stage. Second, it proposes a monitoring, certified-abstention and gain-gated adaptation procedure and isolates which of its parts matter: the risk score and the data used for certification matter, drift-based routing does not, label-shift monitoring is essential, and the accuracy gains of adaptation come mainly from replay and severity-scaled anchoring, with the gate acting chiefly as a safeguard against useless updates. Third, it reports a comparison with six abstention rules, five continual-learning baselines and a range of ablations over two Qwen3 sizes with three seeds each and a second model family, using statistics that do not treat repeated evaluations of the same items as independent, and it states where \sj{} does not help.

\begin{figure}[!htbp]
  \centering
  \resizebox{\textwidth}{!}{
\begin{tikzpicture}[
  font=\sffamily\fontsize{8}{9.6}\selectfont, >=Stealth,
  box/.style={draw, rounded corners=2.5pt, align=center, execute at begin node={\hyphenpenalty=10000\exhyphenpenalty=10000}, minimum height=12mm, inner sep=3pt, line width=0.6pt},
  y1/.style={box, fill=yel!35, draw=yelD}, y2/.style={box, fill=yel!85, draw=yelD, line width=1pt},
  pu/.style={box, fill=pur!22, draw=purD}, bl/.style={box, fill=blu!18, draw=bluD},
  gr/.style={box, fill=grn!30, draw=grnD}, rd/.style={box, fill=red!15, draw=red!60!black},
  te/.style={box, fill=tea!20, draw=teaD, minimum height=14mm}, gy/.style={box, fill=black!6, draw=black!50},
  ar/.style={->, line width=0.8pt, draw=bluD},
  lab/.style={font=\sffamily\bfseries\fontsize{9.5}{11}\selectfont, anchor=west},
  sub/.style={font=\sffamily\fontsize{8}{9.6}\selectfont, anchor=west},
  note/.style={font=\sffamily\fontsize{7}{8.4}\selectfont, text=bluD, fill=white, inner sep=1pt}
]
\node[lab] at (-0.45,1.05) {Stage 1\quad \jev{} decision model};
\node[gy, text width=24mm] (stream) at (1.3,0) {\textbf{Decision request}\\ state $x$, options\\ $o_1,\dots,o_K$};
\node[y1, text width=29mm] (iface) at (4.75,0) {\textbf{Answer interface}\\ numbered options\\ prefill \texttt{Best answer: [}};
\node[y2, text width=31mm] (llm) at (8.55,0) {\textbf{LLM backbone + LoRA}\\ Qwen3-0.6B or 4B\\ (tree loss + KL anchors)};
\node[y1, text width=29mm] (trie) at (12.35,0) {\textbf{Trie scoring}\\ one packed forward pass\\ gives $P(k\mid x)$};
\node[gy, text width=25mm] (sig) at (15.95,0) {\textbf{Signals}\\ $P(k\mid x)$, root mass,\\ embedding $h(x)$};
\draw[ar] (stream) -- (iface); \draw[ar] (iface) -- (llm); \draw[ar] (llm) -- (trie); \draw[ar] (trie) -- (sig);
\draw[black!35, dashed] (-0.45,-1.0) -- (17.25,-1.0);
\node[lab] at (-0.45,-1.4) {Stage 2\quad Reliability layer \textnormal{\fontsize{8}{9.6}\selectfont(Module A: monitoring; Module B: certified abstention)}};
\draw[line width=0.8pt, draw=bluD] (sig.south) -- (15.95,-1.85) -- (-0.3,-1.85) -- (-0.3,-4.85);
\draw[ar] (-0.3,-3.1) -- (0.2,-3.1);
\draw[ar] (-0.3,-4.85) -- (0.2,-4.85);
\node[sub, text=bluD] at (0.2,-2.25) {scoring lane};
\node[bl, text width=33mm] (feat) at (1.95,-3.1) {\textbf{Risk features}\\ uncertainty, trie format,\\ rotation consistency};
\node[bl, text width=33mm, fill=blu!32] (risk) at (5.95,-3.1) {\textbf{Risk score} $r(x)$\\ logistic model fitted on\\ ID + proxy-shift data};
\node[sub, text=purD] at (0.2,-4.0) {routing lane};
\node[pu, text width=33mm] (drift) at (1.95,-4.85) {\textbf{Input drift}\\ $k$-NN distance of the\\ option-invariant $h(x)$};
\node[pu, text width=33mm] (route) at (5.95,-4.85) {\textbf{Router + detectors}\\ regime $g$ or unknown;\\ CUSUM, label-$\chi^2$ alarms};
\draw[ar] (feat) -- (risk);
\draw[ar, draw=purD] (drift) -- (route);
\node[box, draw=bluD, line width=1.1pt, fill=white, text width=37mm, minimum height=31mm] (dec) at (10.95,-3.975) {\textbf{Certified decision}\\[2pt] answer iff $r(x)\le\tau$\\[2pt] $\tau=\tau_g$ for a known regime\\ $\tau=\tau_{\mathrm{shift}}$ if unknown\\ or after an alarm\\[2pt] {\fontsize{7}{8.4}\selectfont $\tau_g$ certified per regime\\ by Learn-then-Test}};
\draw[ar] (risk.east) -- node[note, above]{$r(x)$} (risk.east -| dec.west);
\draw[ar, draw=purD] (route.east) -- node[note, above, text=purD]{$\tau_g$ or $\tau_{\mathrm{shift}}$} (route.east -| dec.west);
\node[gr, text width=22mm, minimum height=10mm] (pred) at (15.95,-3.1) {\textbf{Answer} $\hat y$};
\node[rd, text width=22mm, minimum height=10mm] (abst) at (15.95,-4.85) {\textbf{Abstain / defer}};
\draw[ar, draw=grnD] (pred.west -| dec.east) -- node[note, above, text=grnD]{yes} (pred.west);
\draw[ar, draw=red!60!black] (abst.west -| dec.east) -- node[note, above, text=red!60!black]{no} (abst.west);
\draw[black!35, dashed] (-0.45,-5.95) -- (17.25,-5.95);
\node[lab] at (-0.45,-6.4) {Stage 3\quad Module C: continual adaptation with an expected-gain gate};
\node[gy, text width=24mm] (data) at (1.3,-7.65) {\textbf{New regime data}\\ labelled training\\ + held-out items};
\node[te, text width=29mm] (sev) at (4.75,-7.65) {\textbf{Measure shift}\\ drift and accuracy\\ tests $\rightarrow$ severity $\hat s$};
\node[te, text width=31mm] (adapt) at (8.55,-7.65) {\textbf{Severity-scaled update}\\ new data + replay,\\ anchored to previous model};
\node[te, text width=29mm, fill=tea!45, line width=1pt] (gate) at (12.35,-7.65) {\textbf{Expected-gain gate}\\ keep if held-out gain\\ is significant and old\\ regimes hold, else\\ roll back};
\node[te, text width=25mm] (ref) at (15.95,-7.65) {\textbf{Refresh}\\ add regime $g$ and\\ certify its $\tau_g$};
\draw[ar] (data) -- (sev); \draw[ar] (sev) -- (adapt); \draw[ar] (adapt) -- (gate);
\draw[ar] (gate) -- (ref);
\node[font=\sffamily\fontsize{7}{8.4}\selectfont, text=oraD, anchor=north east] at (17.2,-8.65) {$\circlearrowleft$ the adapter and thresholds $\{\tau_g\}$ are redeployed in Stages 1--2};
\end{tikzpicture}}
  \caption{Overview of \sj. \textbf{Stage~1}: the \jev{} decision model scores numbered options through a prefix-free trie and exposes decision probabilities, the probability mass on legal identifier tokens and a state embedding. \textbf{Stage~2}: in the scoring lane, uncertainty, format and consistency features give the risk score $r(x)$; in the routing lane, an option-invariant drift distance assigns the input to a known regime $g$ or to \emph{unknown} and drives the stream detectors. The model answers only if $r(x)$ is below the threshold certified for that regime, or below the conservative $\tshift$ for unknown inputs and after an alarm. \textbf{Stage~3}: a labelled batch from a new regime drives a severity-scaled LoRA update that is kept only if it significantly improves held-out accuracy without losing more than 5 points on earlier regimes.}
  \label{fig:framework}
\end{figure}

\section{Related Work}
\label{sec:related}

\paragraph{Decision models and LLM judges.}
Jev-style decision models~\citep{almeida2026jev} return calibrated probabilities over typed options and are ranked on JevBench~\citep{jevbench2026}. \citet{li2026llmasjev} extract the same behaviour from open LLMs with an architecture-preserving answer interface, a training-free trie-scoring recipe and a tree-factorized listwise fine-tuning loss with KL anchors on legal-token mass, illegal-token distribution and context positions. Work on LLM-as-a-judge~\citep{zheng2023judge} has shown that judgment \emph{distributions} carry more information than greedy verdicts~\citep{wang2025judgmentdist}, that judges are over-confident~\citep{tian2025overconfidence}, biased~\citep{gao2025judgebias} and benefit from distribution-matching training~\citep{li2025judgingconfidence} or reasoning~\citep{chen2025judgelrm}, and that reported judge accuracies require bias correction~\citep{lee2026reportjudge} and reproducible protocols~\citep{lushtaku2026judgearena}. Confidence routing between a cheap decision model and an expensive judge is effective in-distribution~\citep{li2026jevjudge,zhao2025confdiversity}, but \citet{rao2026jevrubric} find that decision models and LLM judges err on the same items, which limits cascades, and \citet{li2026jevjudge} report that routing weakens on style-adversarial inputs. Multiple-choice selection by LLMs is sensitive to option order and identifiers~\citep{robinson2023mcqa,zheng2024selectors,pezeshkpour2024order}, which motivates our candidate-set shifts and rotation-consistency signal. None of these works studies decision models under distribution shift or provides risk guarantees.

\paragraph{Uncertainty and calibration of LLMs.}
Temperature scaling~\citep{guo2017calibration} remains a strong post-hoc calibrator for pre-trained transformers~\citep{desai2020calibration}. LLM uncertainty can be read from token probabilities~\citep{kadavath2022know}, elicited verbally~\citep{xiong2024uncertainty} or estimated from semantic dispersion~\citep{farquhar2024semantic}; for decision models the categorical output makes probability-based scores natural. UniCR~\citep{oehri2025unicr} fuses heterogeneous evidence into a calibrated correctness probability and enforces an error budget by refusal, and robust-satisficing aggregation~\citep{cetinkaya2026rse} stabilizes multi-LLM decisions under shift. Calibration is known to deteriorate under dataset shift~\citep{ovadia2019trust}.

\paragraph{Selective prediction and conformal risk control.}
Selective classification trades coverage for risk~\citep{elyaniv2010foundations,geifman2017selective}, with the maximum softmax probability as the classical score~\citep{hendrycks2017baseline}. Conformal prediction~\citep{vovk2005algorithmic,angelopoulos2023gentle} and its risk-controlling extensions---RCPS~\citep{bates2021rcps}, conformal risk control~\citep{angelopoulos2024crc} and Learn-then-Test (LTT)~\citep{angelopoulos2021ltt}---give finite-sample guarantees under exchangeability, and Mondrian (group-conditional) variants give per-group guarantees~\citep{vovk2005algorithmic}. For LLMs these tools have been applied to multiple-choice QA~\citep{kumar2023conformalmcqa}, generation~\citep{quach2024conformallm} and factuality~\citep{mohri2024factuality}. Very recent work studies when conformal risk control can certify structured LLM outputs and shows that guarantees fail under dataset shift unless re-verified~\citep{kotte2026crccertify}, proposes similarity-weighted calibration for domain shift~\citep{lin2026dscp}, and enforces risk per group hierarchy~\citep{salem2026hgcrc} or per argument role~\citep{rahman2026rolecrc}. Beyond exchangeability, weighted~\citep{barber2023beyond}, adaptive~\citep{gibbs2021adaptive} and label-shift-aware~\citep{podkopaev2021labelshift} conformal methods exist, and online PAC-Bayes certificates have been proposed for forecasting under shift~\citep{huang2026certificate}. \sj{} is closest to group-conditional risk control: it certifies one LTT threshold per known regime, routes inputs to regimes with an option-invariant drift monitor, and falls back to a conservative threshold for unknown inputs and after stream alarms.

\paragraph{Shift and out-of-distribution detection.}
Dataset-shift detection compares feature or prediction distributions with two-sample tests~\citep{rabanser2019failing}; BBSE~\citep{lipton2018bbse} detects and corrects label shift from a black-box predictor's confusion matrix. Embedding-based OOD scores include Mahalanobis~\citep{lee2018mahalanobis}, relative Mahalanobis~\citep{ren2021rmd}, deep $k$-NN~\citep{sun2022knn} and energy~\citep{liu2020energy}; sequential change detection uses CUSUM~\citep{page1954cusum}. Concept drift in streams~\citep{gama2014survey} and zero-day settings where test data come from unseen distributions~\citep{alsiam2026zeroday} motivate continuous monitoring, and LLM-mediated continuous adaptation has been proposed for anomaly detection streams~\citep{devireddy2025calm}. WILDS~\citep{koh2021wilds} benchmarks in-the-wild shifts but not candidate-set shift, which is specific to decision models.

\paragraph{Continual learning for LLMs.}
Classical remedies for catastrophic forgetting are regularization (EWC~\citep{kirkpatrick2017ewc}), distillation (LwF~\citep{li2017lwf}) and replay~\citep{rolnick2019replay,chaudhry2019tiny,lopezpaz2017gem}; see \citet{wang2024clsurvey} and, for LLMs, \citet{shi2025cllm}. LLMs forget during continual fine-tuning~\citep{luo2023forgetting}, although LoRA forgets less than full fine-tuning~\citep{biderman2024lora}. Recent work schedules replay by parameter change~\citep{feng2026forever} and shows that conformal coverage collapses before accuracy under sequential fine-tuning, fixable by per-task recalibration~\citep{shihab2026continualcal}---an observation that our per-regime recertification (B6) also exploits. The KL anchors of \jev~\citep{li2026llmasjev} act as a stability regularizer; we extend them to previous-stage anchors and couple their strength to measured shift severity.

\Cref{tab:related} positions \sj{} against the closest lines of work.

\begin{table}[!htbp]
\centering\small
\caption{Positioning. \checkmark: addressed; $\circ$: partially; ---: not addressed.}
\label{tab:related}
\resizebox{\columnwidth}{!}{%
\begin{tabular}{lcccccc}\toprule
Work & Decision-model interface & Shift benchmark & Risk guarantee & Per-regime guarantee & Stream detection & Continual adaptation \\\midrule
\jev~\citep{li2026llmasjev} & \checkmark & --- & --- & --- & --- & --- \\
JEV-as-a-Judge~\citep{li2026jevjudge} & \checkmark & $\circ$ & $\circ$ & --- & --- & --- \\
UniCR~\citep{oehri2025unicr} & --- & --- & \checkmark & --- & --- & --- \\
HG-CRC~\citep{salem2026hgcrc} & --- & $\circ$ & \checkmark & \checkmark & --- & --- \\
DS-CP~\citep{lin2026dscp} & --- & \checkmark & $\circ$ & --- & --- & --- \\
Continual calibration~\citep{shihab2026continualcal} & --- & --- & $\circ$ & \checkmark & --- & \checkmark \\
FOREVER~\citep{feng2026forever} & --- & --- & --- & --- & --- & \checkmark \\
BBSE / Failing loudly~\citep{lipton2018bbse,rabanser2019failing} & --- & \checkmark & --- & --- & \checkmark & --- \\
\textbf{\sj{} (ours)} & \checkmark & \checkmark & \checkmark & \checkmark & \checkmark & \checkmark \\
\bottomrule\end{tabular}}
\end{table}


\section{Method}
\label{sec:method}

\subsection{Problem setting}
\label{sec:problem}
A decision instance consists of a state $x$ (instruction and input text), $K$ candidate options and, for training and evaluation, the index $y$ of the correct option. A decision model returns probabilities $P(k\mid x)$ and predicts $\hat y(x)=\arg\max_k P(k\mid x)$. A selective model with risk score $r(x)$ and threshold $\tau$ answers only when $r(x)\le\tau$. On a data distribution $\mathcal D$ its \emph{coverage} is the fraction of answered inputs and its \emph{selective risk} is the error rate among them,
\begin{equation}
  R_{\mathcal D}(\tau)=\Pr_{\mathcal D}\big[\hat y(x)\neq y \,\big|\, r(x)\le\tau\big].
  \label{eq:selrisk}
\end{equation}
The deployment stream passes through regimes $\mathcal D_0$ (in-distribution, ID), $\mathcal D_1,\mathcal D_2,\dots$, and a labelled batch may arrive after each regime. \sj{} aims to keep $R_{\mathcal D_t}(\tau)\le\alpha$ with high probability at the highest possible coverage, and to update the model only when the update improves accuracy on the new regime without degrading earlier ones.

\subsection{Decision model}
\label{sec:jev}
We follow \citet{li2026llmasjev}. Options are listed as \texttt{[1]}, \dots, \texttt{[K]} in the user turn and the assistant turn is pre-filled with ``\texttt{Best answer: [}''. The score of option $k$ is the log-probability of its identifier suffix $c_k=$ ``\texttt{k]}'',
\begin{equation}
  s_k=\log p_\theta(c_k\mid x),\qquad P(k\mid x)=\mathrm{softmax}_k(s_k/T),
  \label{eq:score}
\end{equation}
where $T$ is a temperature fitted on ID validation data. Because the suffixes form a prefix-free trie, all options are scored in one forward pass: we pack the trie nodes after the prompt with a tree-shaped attention mask so that each node attends only to the prompt and its own ancestors, which gives exactly the same scores as scoring each option separately. For fine-tuning we use the base paper's tree-factorized cross-entropy over the legal identifier tokens, plus KL anchors that tie the next-token distribution at the reply start and after the answer to the frozen model, which preserves normal text generation. Adapters are LoRA~\citep{hu2022lora} modules with rank 16 on all linear layers.

\subsection{Module A: monitoring}
\label{sec:moduleA}
Module~A computes four groups of signals for every decision without extra forward passes, except for the consistency check. \emph{Uncertainty} is described by the normalized entropy, the maximum probability, the margin between the two most likely options and the number of options. \emph{Format} is the log of the probability mass that the model places on legal identifier tokens at the trie root; low mass indicates that the model would rather produce something other than an option. \emph{Consistency} compares the decision distribution with the one obtained after cyclically rotating the options, using the total-variation distance and agreement of the two argmax decisions; because it doubles scoring cost, we use it only for the 0.6B model. \emph{Input drift} is the deep $k$-NN distance~\citep{sun2022knn} between the input's state embedding and a reference pool of ID inputs. The embedding is the mean final-layer hidden state over the tokens of the state text only, so it does not change when the options change.

Module~A also monitors the stream of decisions with three detectors: CUSUM~\citep{page1954cusum} on standardized per-decision signals, a windowed Kolmogorov--Smirnov test, and a black-box label-shift test in the spirit of BBSE~\citep{lipton2018bbse}, which compares the histogram of \emph{predicted} labels in a 50-item window with its reference frequencies by a $\chi^2$ test. Every alarm level is set by simulation so that at most 5\% of no-shift streams raise a false alarm.

\subsection{Module B: certified, regime-conditional abstention}
\label{sec:moduleB}
The risk score $r(x)$ is a logistic regression on the uncertainty, format and consistency signals, trained to predict errors on ID calibration data and on two proxy shifts (SciQ and Emotion) that are disjoint from all test shifts. Each calibration set is split once into two halves: the risk model is fitted, and candidate thresholds are proposed, on the first half only, and thresholds are certified on the second half only. Temperature scaling and the drift reference use separate ID validation data.

Drift distances are deliberately excluded from $r(x)$. A risk model fitted on two proxy shifts learns that ``far from the reference data'' means ``likely wrong'', which holds for some shifts, such as intent sets with many options, and fails for others, such as an easy emerging-class set or knowledge questions that resemble ID data but are answered wrongly with high confidence. Drift instead tells the system \emph{which population} an input comes from, so \sj{} uses it for routing. Each known regime $g$ has its own reference pool and threshold $\tau_g$. An input is assigned to the regime whose normalized $k$-NN distance is smallest, provided this distance is below the regime's 95th percentile; otherwise it is \emph{unknown}. Unknown inputs, and all inputs after a stream alarm, use the conservative threshold $\tshift$, the smaller of the thresholds certified on ID and on proxy-shift data, so a detected shift can only make the model more cautious. An earlier design that placed drift inside $r(x)$ is kept as an ablation.

Thresholds are certified by fixed-sequence Learn-then-Test (LTT)~\citep{angelopoulos2021ltt}. Candidate thresholds are proposed on the first calibration half, ordered from strict to permissive, and tested in that order on the second half: a threshold passes if the one-sided Clopper--Pearson~\citep{clopper1934} upper bound on its selective risk, at confidence $1-\delta$ with $\delta=0.1$, is below $\alpha$. Testing stops at the first failure and the last passing threshold is used, or the model abstains on everything if none passes. Because the risk model and the candidate thresholds depend only on the first half, the second half is used for a single test, and if calibration and test data of a regime are exchangeable the procedure guarantees $R_{\mathcal D_g}(\tau_g)\le\alpha$ with probability at least $1-\delta$ (Appendix~\ref{app:theory}). No such guarantee is possible for unseen shifts, so for them we report empirical violation rates.

\subsection{Module C: gain-gated continual adaptation}
\label{sec:moduleC}
When a labelled batch from a new regime arrives, Module~C first measures the shift. A binomial test checks whether more new inputs than expected lie beyond the reference 95th-percentile drift distance, and a Fisher exact test checks whether accuracy on the new data is lower than on the reference data. Their effect sizes give a severity $\hat s\in[0,1]$.

The adapter is then trained for one epoch on the new data, with replay of buffered examples from earlier regimes. The buffer keeps, per regime, correctly classified examples with the smallest decision margin together with random examples. Two terms protect earlier behaviour. Tree-local distillation keeps the previous model's option distributions on buffered examples~\citep{hinton2015distilling,li2017lwf}, and the \jev{} KL anchors are computed against the \emph{previous-stage} model instead of the original one. Severity sets the stability--plasticity balance: the anchor weight decreases as $\lambda_{\mathrm{KL}}=\lambda_0e^{-\kappa\hat s}$ with $\kappa=1.5$, while the distillation weight increases with $\hat s$, so that larger shifts allow more change on new inputs and more protection of old decisions.

The update is then judged by an \emph{expected-gain gate}, which asks whether the update makes the model better rather than whether the model has become worse. Halfway through training, a probe compares the checkpoint with the current model on held-out new-regime items. If no shift was detected and the checkpoint wins no more items than it loses, training stops. Otherwise the half-trained and fully trained candidates are each accepted only if a one-sided sign test on paired held-out correctness against the current model gives $p<0.10$ and their accuracy on buffered old-regime items drops by at most 5 points. The accepted candidate with the best combined new- and old-regime accuracy is deployed. If neither is accepted, the update is rolled back. In both cases the regime joins the replay buffer and the reference pools and receives its own certified threshold (system B6). The sign test bounds the probability of accepting an update that does not help by 0.10 per candidate (Appendix~\ref{app:theory}, which also gives the full procedure as pseudocode).

\subsection{Compared systems}
\label{sec:systems}
We compare six systems. \textbf{B1} is the frozen \jev{} model and \textbf{B2} is B1 fine-tuned on an ID mixture. \textbf{B3} adds temperature scaling~\citep{guo2017calibration} to B2. \textbf{B4} adds an uncertainty-only risk score with a single LTT threshold certified on pooled ID and proxy data. \textbf{B5} is the full \sj{} reliability layer with routing and per-regime thresholds, and \textbf{B6} is B5 after continual adaptation and per-regime recertification. The abstention baselines are the max-probability (MSP) score~\citep{hendrycks2017baseline} and the energy score~\citep{liu2020energy}, both with the same pooled LTT certification as B4, and split-conformal LAC prediction sets~\citep{vovk2005algorithmic}, calibrated on all ID and proxy calibration data, that answer only when the set contains a single option. LAC guarantees marginal coverage of the prediction set rather than selective risk. The continual-learning baselines are naive sequential fine-tuning, EWC~\citep{kirkpatrick2017ewc}, sequential fine-tuning with the base paper's KL anchors to the original model (Seq-FT+KL), experience replay~\citep{rolnick2019replay} and LwF~\citep{li2017lwf}. All baselines start from the same B2 adapter and use the same data, steps and LoRA capacity.

\section{Experimental Setup}
\label{sec:setup}

\Cref{tab:benchmark} summarizes the benchmark. The ID tasks are AG News~\citep{zhang2015agnews}, SST-2~\citep{socher2013sst}, BoolQ~\citep{clark2019boolq} and ARC-Challenge~\citep{clark2018arc}, each split into training (500 items per task for the 0.6B model and 250 for the 4B model), validation (100), calibration (200) and test (300) sets. SciQ~\citep{welbl2017sciq} and Emotion~\citep{saravia2018emotion} serve only as proxy shifts for calibration. The test shifts are domain shift to Banking77 with 77 intents~\citep{casanueva2020banking}, Twitter financial-news topics with 20 classes~\citep{twitterfin2022} and MMLU~\citep{hendrycks2021mmlu}; class-prior shift on AG News with imbalance ratios of 10 and 100; emerging classes on CLINC150~\citep{larson2019clinc}, with 15 candidate intents plus ``none of the above'' and 50\% out-of-scope requests; and four candidate-set perturbations of ID items: option rotation, label paraphrase and an added near-duplicate distractor, which leave the task unchanged, and removal of the correct option, which makes every answer wrong. External benchmarks follow the base paper: BoolQ, MMLU, MMLU-Pro~\citep{wang2024mmlupro}, ARC-Challenge, WinoGrande~\citep{sakaguchi2021winogrande}, SciQ and Banking77.

\begin{table}[!htbp]
\centering\small
\caption{Evaluation environments and the use of their data: training (Tr), temperature fitting (Val), risk-model fitting and certification (Cal) and testing (Test).}
\label{tab:benchmark}
\resizebox{\columnwidth}{!}{%
\begin{tabular}{llllr}\toprule
Group & Environment & Source & Role & $n$ (test) \\\midrule
ID & AG News, SST-2, BoolQ, ARC-C & 4 tasks & Tr / Val / Cal / Test & 1{,}200 \\
Proxy shift & SciQ, Emotion & 2 tasks & Cal only & --- \\
A. Domain & Banking77 ($K=77$), Twitter-Fin ($K=20$), MMLU ($K=4$) & 3 tasks & Test (+ adaptation) & $3\times500$ \\
B. Class prior & AG News, $\rho=10$ and $\rho=100$ & resampled & Test & $2\times300$ \\
C. Emerging classes & CLINC150 + ``none of the above'' ($K=16$) & 50\% out of scope & Test (+ adaptation) & 500 \\
D. Candidate set & rotation, paraphrase, distractor / gold removed & ID test items & Test & 153--204 each \\
External & BoolQ, MMLU, MMLU-Pro, ARC-C, WinoGrande, SciQ, Banking77 & paper sizes & Test & 6{,}456--7{,}456 \\
\bottomrule\end{tabular}}
\end{table}

Each detection stream contains 200 ID items followed by 200 items of which 25\%, 50\% or 100\% come from a shifted environment, and we simulate 50 streams per environment and severity. The continual protocol starts from the first-seed B2 adapter and visits Banking77, a control stage of fresh ID data in which adaptation cannot help, Twitter-Fin and CLINC150. Each stage supplies 300 training items and 260 held-out items: 80 for the gate, 20 for buffer validation and 160 for threshold certification. Accuracy on every regime seen so far is measured after each stage on 400 (0.6B) or 300 (4B) evaluation items per regime, and we report final average accuracy, backward transfer (BWT) and forgetting~\citep{lopezpaz2017gem,chaudhry2019tiny}.

The backbones are Qwen3-0.6B~\citep{yang2025qwen3}, the small model of the base paper, Qwen3-4B-Instruct-2507 and the frozen Phi-4-mini-instruct~\citep{abouelenin2025phi4mini} as a second model family. We use Qwen3-4B-Instruct instead of the base paper's Qwen3.5-4B because the latter's hybrid linear attention is impractically slow on T4 GPUs without custom kernels. Training follows the base paper: LoRA rank 16 and $\alpha=32$ on all linear layers, AdamW with learning rate $10^{-4}$, one epoch, batch size 8 and anchor weight 1. Fine-tuning, abstention and continual-learning experiments on the Qwen backbones use three seeds (20260922, 13 and 42); the stream, perturbation, post-adaptation and base-recipe ablation analyses use the first seed, and the frozen models are deterministic. We report means and standard deviations over seeds and paired bootstrap intervals for differences in the area under the risk--coverage curve (AURC). Evaluation items are shared across seeds, and all continual runs start from the same B2 adapter, so we do not pool seeds as independent pairs. Instead, for continual learning we report the per-seed paired difference in accuracy, the number of seeds in which a method is better, and a two-sided exact sign test over unique items on correctness averaged over the three seeds. Abstention is evaluated by coverage, realized selective risk and the number of (environment, seed) tests in which the realized risk exceeds $\alpha$ (\emph{violations}); because these tests share items, we treat the counts as descriptive and also report them per seed. Each of the three parts of the study runs in a single Kaggle session with two T4 GPUs in under 12 hours (8.1, 9.4 and 11.1~h). \Cref{app:impl} lists all implementation details and hyperparameters.

\section{Results}
\label{sec:results}

\subsection{Reproduction of the decision model}
\label{sec:repro}
\Cref{tab:repro} and \cref{fig:repro} compare our implementation with the base paper. Training-free Qwen3-0.6B reaches a six-task external macro accuracy of 57.10 against 57.85 reported, with per-task differences within 2.2 points except SciQ ($-3.5$). Fine-tuning reaches $59.61\pm0.23$ over three seeds, above the reported 58.30, while Banking77 stays at the training-free level (20.4\%) because Banking77 is one of our test shifts and is therefore excluded from our training mixture. Fine-tuning halves the external calibration error (ECE 0.246 to 0.120), and text-mode decoding agrees with the scored decision on all audited items.

\begin{table}[!htbp]\centering\small
\caption{Reproduction of \jev{} with Qwen3-0.6B on the base paper's external benchmarks and sample sizes (accuracy, \%). ECE averages the six macro tasks.}\label{tab:repro}
\resizebox{\textwidth}{!}{%
\begin{tabular}{lrrrrrrrrr}\toprule
System & BoolQ & MMLU & MMLU-Pro & ARC-C & WinoGr. & SciQ & Macro-6 & Banking77 & ECE \\\midrule
Base paper, training-free & 72.60 & 46.75 & 22.25 & 54.17 & 53.60 & 97.75 & 57.85 & 22.21 & -- \\
\textbf{This study, training-free (B1)} & 73.80 & 45.17 & 24.00 & 54.00 & 51.40 & 94.25 & 57.10 & 20.48 & 0.246 \\
Base paper, LoRA $\lambda$=1 & 72.40 & 46.00 & 20.80 & 57.80 & 55.00 & 97.80 & 58.30 & 54.90 & -- \\
Base paper, LoRA $\lambda$=0.01 & 72.80 & 47.00 & 21.50 & 57.20 & 55.20 & 97.50 & 58.53 & 59.00 & -- \\
Base paper, full FT $\lambda$=1 & 74.40 & 47.83 & 20.50 & 60.33 & 53.60 & 97.75 & 59.07 & 60.49 & -- \\
\textbf{This study, LoRA $\lambda$=1 (B2, 3 seeds)} & 77.87 & 46.19 & 23.71 & 61.83 & 53.00 & 95.08 & 59.61 $\pm$ 0.23 & 20.41 & 0.120 \\
\bottomrule\end{tabular}}
\end{table}

\begin{figure}[!htbp]
  \centering
  \includegraphics[width=\textwidth]{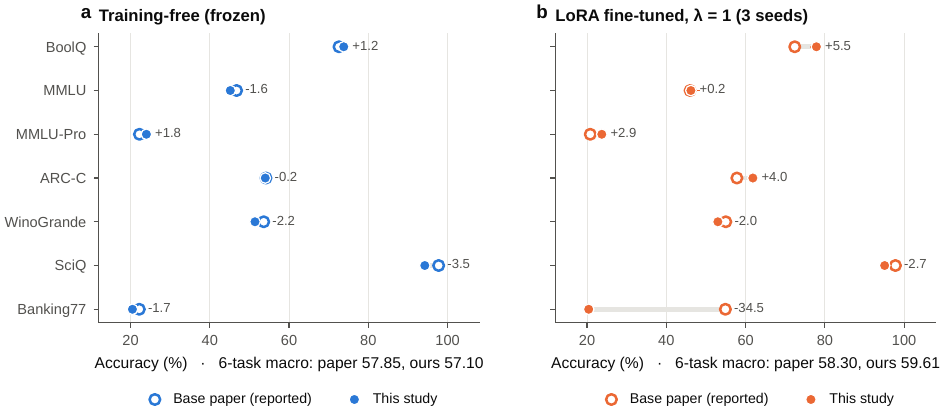}
  \caption{Per-task reproduction of the base paper on Qwen3-0.6B. Hollow markers: reported; filled: this study (b: mean and range over 3 seeds). The Banking77 gap in (b) reflects intent data in the base paper's training mixture, which our ID mixture excludes by design.}
  \label{fig:repro}
\end{figure}

\subsection{RQ1: decision models under shift}
\label{sec:rq1}
Domain shift to intent sets with many options is the most damaging shift (\cref{tab:rq1,tab:main}, \cref{fig:calib}). Frozen Qwen3-0.6B falls from 73.3\% ID accuracy to 18.4\% on Banking77 and 27.6\% on Twitter-Fin while remaining confident (ECE 0.47 and 0.58), and Qwen3-4B falls from 91.1\% to 55.4\% and 60.2\%. Class-prior shift does not reduce accuracy at either scale (AG News is easier than the ID average), and emerging classes cost 4--5 points because ``none of the above'' is often chosen correctly. Candidate changes that leave the task unchanged still hurt the small model: with an added near-duplicate distractor, Qwen3-0.6B reaches 59.3\% against 73.3\% on ID data, while Qwen3-4B reaches 89.7\% against 91.1\%. Fine-tuning with KL anchors and temperature scaling (B3) reduces the shift-averaged ECE from 0.289 to 0.087 for Qwen3-0.6B and from 0.182 to 0.078 for Qwen3-4B, and raises the 0.6B model's shift accuracy from 57.5\% to 64.2\%. These gains come from the base recipe and temperature scaling; the \sj{} reliability layer does not change predictions or calibration and is evaluated separately below. Fine-tuning alone (B2, shift ECE 0.126 on both backbones) accounts for most of the calibration gain on Qwen3-0.6B and about half of it on Qwen3-4B. For Qwen3-4B fine-tuning leaves shift accuracy unchanged (80.4\% vs.\ 80.1\%), consistent with the base paper's observation that gains shrink for stronger backbones. Even after fine-tuning, ECE remains higher under domain shift (0.08--0.17) than on ID data (0.02--0.04).

\begin{table}[!htbp]\centering\small
\caption{RQ1. Accuracy (\%) and ECE of the frozen model (B1, raw) and the fine-tuned, temperature-scaled model (B3; mean $\pm$ s.d.\ over 3 seeds) per environment.}\label{tab:rq1}
\resizebox{\textwidth}{!}{%
\begin{tabular}{llrrrrrrrrr}\toprule
 & & & \multicolumn{4}{c}{Qwen3-0.6B} & \multicolumn{4}{c}{Qwen3-4B-Instruct} \\\cmidrule(lr){4-7}\cmidrule(lr){8-11}
Environment & Shift & $n$ & Acc B1 & ECE B1 & Acc B3 & ECE B3 & Acc B1 & ECE B1 & Acc B3 & ECE B3 \\\midrule
In-distribution & --- & 1200 & 73.3 & 0.178 & 78.5 $\pm$ 1.1 & 0.018 $\pm$ 0.004 & 91.1 & 0.084 & 91.4 $\pm$ 0.3 & 0.038 $\pm$ 0.003 \\
Banking77 & Domain & 500 & 18.4 & 0.472 & 19.2 $\pm$ 0.9 & 0.108 $\pm$ 0.011 & 55.4 & 0.400 & 56.3 $\pm$ 0.2 & 0.128 $\pm$ 0.034 \\
Twitter-Fin & Domain & 500 & 27.6 & 0.582 & 38.2 $\pm$ 0.7 & 0.114 $\pm$ 0.012 & 60.2 & 0.375 & 59.5 $\pm$ 0.8 & 0.167 $\pm$ 0.009 \\
MMLU & Domain & 500 & 43.4 & 0.346 & 46.5 $\pm$ 1.1 & 0.095 $\pm$ 0.017 & 71.0 & 0.260 & 71.3 $\pm$ 0.3 & 0.077 $\pm$ 0.008 \\
AG prior $\rho$=10 & Class prior & 300 & 81.7 & 0.136 & 90.1 $\pm$ 1.0 & 0.044 $\pm$ 0.015 & 92.3 & 0.078 & 91.3 $\pm$ 0.9 & 0.044 $\pm$ 0.006 \\
AG prior $\rho$=100 & Class prior & 300 & 78.0 & 0.187 & 89.0 $\pm$ 3.6 & 0.048 $\pm$ 0.007 & 86.3 & 0.134 & 87.4 $\pm$ 0.5 & 0.049 $\pm$ 0.006 \\
CLINC-OOS & Emerging classes & 500 & 68.0 & 0.189 & 71.5 $\pm$ 3.2 & 0.153 $\pm$ 0.020 & 86.8 & 0.119 & 86.2 $\pm$ 0.8 & 0.066 $\pm$ 0.017 \\
Option rotation & Candidate set & 204 & 70.6 & 0.199 & 76.8 $\pm$ 0.6 & 0.070 $\pm$ 0.011 & 92.6 & 0.076 & 89.9 $\pm$ 0.3 & 0.049 $\pm$ 0.007 \\
Label paraphrase & Candidate set & 153 & 70.6 & 0.209 & 77.6 $\pm$ 0.8 & 0.074 $\pm$ 0.005 & 89.5 & 0.102 & 89.8 $\pm$ 0.4 & 0.047 $\pm$ 0.003 \\
Near-duplicate distractor & Candidate set & 204 & 59.3 & 0.278 & 69.3 $\pm$ 3.0 & 0.077 $\pm$ 0.014 & 89.7 & 0.095 & 89.1 $\pm$ 0.7 & 0.075 $\pm$ 0.005 \\
SciQ & External & 300 & 92.7 & 0.059 & 93.8 $\pm$ 0.8 & 0.065 $\pm$ 0.013 & 98.0 & 0.020 & 98.2 $\pm$ 0.2 & 0.034 $\pm$ 0.003 \\
\bottomrule\end{tabular}}
\end{table}

\begin{figure}[!htbp]
  \centering
  \includegraphics[width=\textwidth]{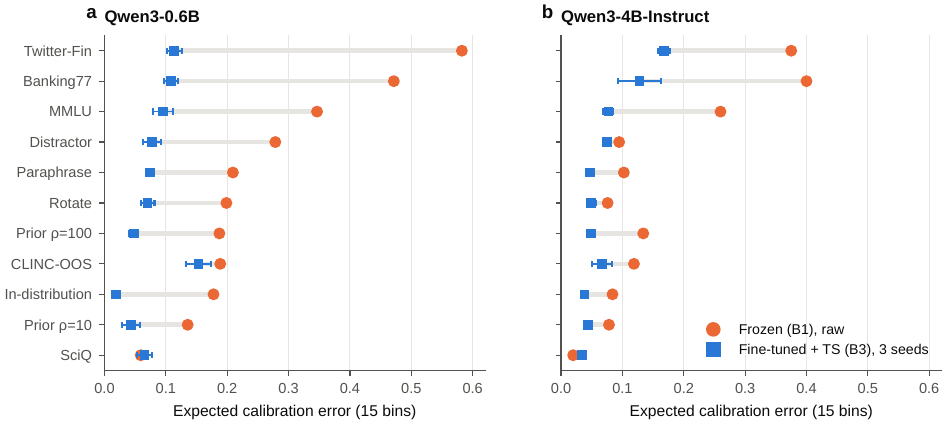}
  \caption{Expected calibration error per environment for the frozen model (B1) and the fine-tuned, temperature-scaled model (B3; mean and s.d.\ over 3 seeds): (a)~Qwen3-0.6B, (b)~Qwen3-4B-Instruct.}
  \label{fig:calib}
\end{figure}

\begin{table}[!htbp]\centering\small
\caption{Systems B1--B6. Shift metrics average the nine shifted environments; abstention at $\alpha=0.10$, $\delta=0.1$. The calibration gains of B2 and B3 come from fine-tuning with KL anchors and temperature scaling; the reliability layer (B4, B5) does not change predictions or calibration. B6 is evaluated on the continual evaluation sets: ID accuracy after the last stage (3 seeds) and, marked $^\dagger$, coverage and realized risk on the ID regime with its recertified threshold (seed 1; \cref{tab:post}).}\label{tab:main}
\resizebox{\textwidth}{!}{%
\begin{tabular}{llrrrrrrrr}\toprule
Backbone & System & Seeds & ID acc (\%) & ID ECE & Shift acc (\%) & Shift ECE & Coverage@.10 (\%) & Risk@.10 (\%) & Violations@.10 \\\midrule
Qwen3-0.6B & B1 frozen & 1 & 73.3 & 0.178 & 57.5 & 0.289 & -- & -- & -- \\
 & B2 fine-tuned + KL anchors & 3 & 78.5 $\pm$ 1.1 & 0.061 $\pm$ 0.004 & 64.2 $\pm$ 0.6 & 0.126 $\pm$ 0.007 & -- & -- & -- \\
 & B3 = B2 + temperature scaling & 3 & 78.5 $\pm$ 1.1 & 0.018 $\pm$ 0.004 & 64.2 $\pm$ 0.6 & 0.087 $\pm$ 0.003 & -- & -- & -- \\
 & B4 = B3 + uncertainty abstention & 3 & \multicolumn{4}{c}{\emph{as B3}} & 37.9 $\pm$ 4.9 & 6.4 $\pm$ 1.9 & 5/27 \\
 & B5 = B3 + \sj{} monitor & 3 & \multicolumn{4}{c}{\emph{as B3}} & 33.3 $\pm$ 6.6 & 5.2 $\pm$ 2.8 & 4/27 \\
 & B6 = B5 after continual adaptation & 3 / 1 & 77.6 $\pm$ 0.9 & -- & \multicolumn{2}{c}{\emph{see \cref{tab:cl,tab:cl4b}}} & 64.5$^\dagger$ & 11.6$^\dagger$ & -- \\
Qwen3-4B & B1 frozen & 1 & 91.1 & 0.084 & 80.4 & 0.182 & -- & -- & -- \\
 & B2 fine-tuned + KL anchors & 3 & 91.4 $\pm$ 0.3 & 0.033 $\pm$ 0.022 & 80.1 $\pm$ 0.2 & 0.126 $\pm$ 0.024 & -- & -- & -- \\
 & B3 = B2 + temperature scaling & 3 & 91.4 $\pm$ 0.3 & 0.038 $\pm$ 0.003 & 80.1 $\pm$ 0.2 & 0.078 $\pm$ 0.005 & -- & -- & -- \\
 & B4 = B3 + uncertainty abstention & 3 & \multicolumn{4}{c}{\emph{as B3}} & 55.8 $\pm$ 2.0 & 4.8 $\pm$ 0.6 & 3/27 \\
 & B5 = B3 + \sj{} monitor & 3 & \multicolumn{4}{c}{\emph{as B3}} & 60.1 $\pm$ 7.4 & 6.2 $\pm$ 2.8 & 6/27 \\
 & B6 = B5 after continual adaptation & 3 / 1 & 89.0 $\pm$ 0.9 & -- & \multicolumn{2}{c}{\emph{see \cref{tab:cl,tab:cl4b}}} & 85.7$^\dagger$ & 4.7$^\dagger$ & -- \\
\bottomrule\end{tabular}}
\end{table}

\subsection{RQ2: certified abstention under shift}
\label{sec:rq2}
\Cref{tab:abst_small,tab:abst_large} report every abstention rule on both backbones, \cref{tab:routing} isolates the threshold rule from the risk score, and \cref{fig:riskcontrol} summarizes the comparison. We focus on $\alpha=0.10$. At $\alpha=0.05$, every rule whose threshold is certified on ID data alone, including B5, certifies no threshold in most runs and abstains on all shifted inputs (coverage 0.0\% on Qwen3-0.6B and 1.4\% on Qwen3-4B); B5 also abstains on all ID inputs in five of six runs. Its zero violations at $\alpha=0.05$ are therefore vacuous. With 60--350 answered calibration points and $\delta=0.1$, the Clopper--Pearson bound falls below 0.05 only if the observed error rate is below about 2--3\%; in four of the six Qwen runs no candidate threshold met this on the ID calibration half, in one the strictest candidate failed and ended the fixed sequence although later candidates would have passed, and in one a threshold answering 17\% of ID inputs was certified (\cref{sec:limits}). Pooled certification, with 700 instead of 400 test-half points, is less affected.

The risk score matters even when the certification procedure is the same. With the same pooled LTT threshold, the MSP score violates $\alpha=0.10$ in 15 of 27 tests on Qwen3-0.6B (five in every seed) at a mean realized risk of 16.0\%, whereas the learned uncertainty score of B4 violates it in 5 tests at 6.4\%; on Qwen3-4B the counts are 5 and 3. A plausible reason is that one threshold on the maximum probability does not transfer across tasks with very different numbers of options, whereas the learned score also uses the margin and the number of options. Split-conformal LAC sets answer the most inputs (53.6\% and 73.8\%), but at realized risks of 24.5\% and 12.1\%, and violate 14 and 9 tests, which is expected because they control marginal coverage rather than selective risk. Adding the format and consistency signals to the learned score does not reduce violations under a pooled threshold (8 and 3 tests at $\alpha=0.10$).

Drift-based routing has little effect (\cref{tab:routing}). Routing changes a threshold only when the proxy-shift threshold is stricter than the ID threshold, which never happened at $\alpha=0.05$ and happened in one to three of the three runs at $\alpha=0.10$ and 0.20. B5 therefore has the same violation counts as the same score with an ID-only threshold on Qwen3-4B, and differs from it only at $\alpha=0.20$ on Qwen3-0.6B (12 against 14 violations). The calibration data matter more than routing, although not uniformly: certifying the uncertainty score on ID data instead of pooled data changes the violations at $\alpha=0.10$ from 5 to 4 on Qwen3-0.6B but from 3 to 6 on Qwen3-4B. Applying the stricter of the ID and proxy thresholds to every input gives the fewest violations on both backbones (3 and 2 of 27 at $\alpha=0.10$; 7 and 4 summed over the three levels for the \sj{} score), at a coverage cost of 7.4 points on Qwen3-4B at $\alpha=0.10$ (52.7\% against 60.1\%) and 10.4 points on Qwen3-0.6B at $\alpha=0.20$. Because this rule emerged from the post-hoc analysis in \cref{tab:routing}, we report it as a finding to be confirmed rather than as part of the method. At $\alpha=0.10$, B5 answers 33.3\% of shifted inputs at 5.2\% realized risk on Qwen3-0.6B and 60.1\% at 6.2\% on Qwen3-4B, with 4 and 6 violations, comparable to B4 (5 and 3). The risk--coverage curves (\cref{fig:riskcov}) explain why no rule can answer many Banking77 or Twitter-Fin inputs: even the 5\% most confident decisions of Qwen3-0.6B on these sets are wrong 30--60\% of the time.

Placing the drift distances inside the risk model, as in our earlier design, increases summed violations relative to B5 from 16 to 20 on Qwen3-0.6B and from 9 to 19 on Qwen3-4B, worsens AURC on Qwen3-0.6B (0.2332 against 0.2265) and makes coverage unstable across seeds ($27.5\pm24.0\%$). A risk model using drift alone ranks errors barely above chance (error AUROC 0.54 and 0.50). The drift embedding is nevertheless useful for routing in the sense that none of the candidate-perturbed items is routed to $\tshift$ (\cref{tab:pert}), because their state text is unchanged, and it drives the stream detectors of \cref{sec:streams}.

The different risk scores rank errors almost identically. Shift AURC lies between 0.2244 and 0.2307 on Qwen3-0.6B and between 0.0902 and 0.0958 on Qwen3-4B for the MSP, uncertainty and \sj{} scores and their ablations without format or consistency signals, error AUROC is 0.75--0.76 and 0.80--0.81, and paired bootstrap tests do not separate B4 and B5 on most environments. Removing rotation consistency raises AURC to 0.2307, while removing the format signal lowers it to 0.2244 and reduces violations to 11, so the format signal does not help. The energy score is unsuitable for trie scores: because each option's score is already a normalized log-probability, the energy measures total identifier mass rather than confidence among options, its error AUROC is inverted (0.25 and 0.20) and its certification abstains on everything.

When the correct option is removed, every answer is wrong. B5 abstains on 40.0\% of such items on Qwen3-0.6B and 41.3\% on Qwen3-4B, against 35.8\% and 17.2\% for B4, but the Qwen3-4B figure varies strongly between seeds (72\%, 14\% and 38\%). Under the three task-preserving perturbations B5 keeps realized risk between 2.6\% and 8.8\% (seed 1, \cref{tab:pert}).

\begin{table}[!htbp]\centering\small
\caption{Abstention on B2 under shift (Qwen3-0.6B; 3 seeds $\times$ 9 shifted environments = 27 tests per $\alpha$). ``Pooled'' thresholds are certified by LTT on ID and proxy-shift calibration data together, ``ID-only'' on ID data alone; B5 uses the ID threshold and routes unknown inputs to $\tau_{\mathrm{shift}}$. Coverage and realized risk average the nine environments (mean $\pm$ s.d.\ over seeds); a coverage of 0 means that no threshold could be certified and the rule abstains on everything. Violations count (environment, seed) tests whose realized risk exceeds $\alpha$; because the test items are shared across seeds, the counts are descriptive. MSP and LAC rank inputs by the same score and therefore have the same AURC.}\label{tab:abst_small}
\resizebox{\textwidth}{!}{%
\begin{tabular}{lrrrrrrr}\toprule
Risk score / threshold rule & Shift AURC $\downarrow$ & Cov@.05 (\%) & Cov@.10 (\%) & Risk@.10 (\%) & Violations $\alpha$=.05/.10/.20 & Viol.@.10 per seed & Gold removed: abstain (\%) \\\midrule
MSP score, pooled LTT $\tau$ & 0.2287 $\pm$ 0.0034 & 28.8 $\pm$ 12.5 & 43.4 $\pm$ 2.3 & 16.0 $\pm$ 1.2 & 12 / 15 / 18 & 5 / 5 / 5 & 35.5 $\pm$ 1.5 \\
Energy score, pooled LTT $\tau$ & 0.4852 $\pm$ 0.0139 & 0.0 & 0.0 & 0.0 & 0 / 0 / 0 & 0 / 0 / 0 & 100.0 \\
Split-conformal LAC (not risk-certified) & 0.2287 $\pm$ 0.0034 & 36.3 $\pm$ 2.0 & 53.6 $\pm$ 0.9 & 24.5 $\pm$ 0.4 & 17 / 14 / 15 & 5 / 5 / 4 & 35.6 $\pm$ 0.6 \\
\sj{} risk score, pooled LTT $\tau$ & 0.2265 $\pm$ 0.0036 & 18.0 $\pm$ 16.7 & 39.1 $\pm$ 0.9 & 7.5 $\pm$ 1.0 & 4 / 8 / 10 & 2 / 3 / 3 & 35.8 $\pm$ 3.0 \\
B4: uncertainty risk, pooled LTT $\tau$ & 0.2280 $\pm$ 0.0027 & 20.2 $\pm$ 17.6 & 37.9 $\pm$ 4.9 & 6.4 $\pm$ 1.9 & 4 / 5 / 13 & 2 / 2 / 1 & 35.8 $\pm$ 2.9 \\
Uncertainty risk, ID-only LTT $\tau$ & 0.2280 $\pm$ 0.0027 & 0.0 & 35.4 $\pm$ 5.2 & 5.3 $\pm$ 1.7 & 0 / 4 / 8 & 2 / 1 / 1 & 37.9 $\pm$ 3.6 \\
Drift features inside the risk score & 0.2332 $\pm$ 0.0036 & 9.3 $\pm$ 16.1 & 27.5 $\pm$ 24.0 & 5.0 $\pm$ 4.6 & 2 / 4 / 14 & 1 / 0 / 3 & 55.1 $\pm$ 39.3 \\
B5 $-$ format signal & 0.2244 $\pm$ 0.0007 & 0.0 & 30.3 $\pm$ 3.5 & 3.5 $\pm$ 1.1 & 0 / 1 / 10 & 0 / 0 / 1 & 43.1 $\pm$ 2.5 \\
B5 $-$ rotation consistency & 0.2307 $\pm$ 0.0043 & 0.0 & 34.0 $\pm$ 7.8 & 6.6 $\pm$ 3.1 & 0 / 5 / 9 & 3 / 0 / 2 & 38.6 $\pm$ 5.8 \\
Drift features only & 0.3481 $\pm$ 0.0127 & 0.0 & 0.0 & 0.0 & 0 / 0 / 3 & 0 / 0 / 0 & 100.0 \\
\textbf{B5: \sj{} risk, ID LTT $\tau$ + routing} & 0.2265 $\pm$ 0.0036 & 0.0 & 33.3 $\pm$ 6.6 & 5.2 $\pm$ 2.8 & 0 / 4 / 12 & 1 / 0 / 3 & 40.0 $\pm$ 6.9 \\
\bottomrule\end{tabular}}
\end{table}

\begin{table}[!htbp]\centering\small
\caption{Abstention on B2 under shift (Qwen3-4B-Instruct-2507; 3 seeds $\times$ 9 shifted environments = 27 tests per $\alpha$). ``Pooled'' thresholds are certified by LTT on ID and proxy-shift calibration data together, ``ID-only'' on ID data alone; B5 uses the ID threshold and routes unknown inputs to $\tau_{\mathrm{shift}}$. Coverage and realized risk average the nine environments (mean $\pm$ s.d.\ over seeds); a coverage of 0 means that no threshold could be certified and the rule abstains on everything. Violations count (environment, seed) tests whose realized risk exceeds $\alpha$; because the test items are shared across seeds, the counts are descriptive. MSP and LAC rank inputs by the same score and therefore have the same AURC. The 4B model has no consistency signal, so B5 without the format signal reduces to the uncertainty risk with routing and is omitted.}\label{tab:abst_large}
\resizebox{\textwidth}{!}{%
\begin{tabular}{lrrrrrrr}\toprule
Risk score / threshold rule & Shift AURC $\downarrow$ & Cov@.05 (\%) & Cov@.10 (\%) & Risk@.10 (\%) & Violations $\alpha$=.05/.10/.20 & Viol.@.10 per seed & Gold removed: abstain (\%) \\\midrule
MSP score, pooled LTT $\tau$ & 0.0902 $\pm$ 0.0021 & 37.3 $\pm$ 5.4 & 55.5 $\pm$ 6.8 & 7.7 $\pm$ 0.6 & 6 / 5 / 9 & 2 / 1 / 2 & 34.6 $\pm$ 2.7 \\
Energy score, pooled LTT $\tau$ & 0.3417 $\pm$ 0.0057 & 0.0 & 0.0 & 0.0 & 0 / 0 / 9 & 0 / 0 / 0 & 100.0 \\
Split-conformal LAC (not risk-certified) & 0.0902 $\pm$ 0.0021 & 39.3 $\pm$ 4.0 & 73.8 $\pm$ 4.6 & 12.1 $\pm$ 0.7 & 13 / 9 / 7 & 3 / 3 / 3 & 27.0 $\pm$ 3.4 \\
\sj{} risk score, pooled LTT $\tau$ & 0.0948 $\pm$ 0.0063 & 50.1 $\pm$ 1.5 & 55.9 $\pm$ 1.9 & 4.7 $\pm$ 0.4 & 6 / 3 / 3 & 1 / 1 / 1 & 48.4 $\pm$ 33.3 \\
B4: uncertainty risk, pooled LTT $\tau$ & 0.0958 $\pm$ 0.0047 & 50.5 $\pm$ 0.7 & 55.8 $\pm$ 2.0 & 4.8 $\pm$ 0.6 & 7 / 3 / 4 & 1 / 1 / 1 & 17.2 $\pm$ 4.7 \\
Uncertainty risk, ID-only LTT $\tau$ & 0.0958 $\pm$ 0.0047 & 1.4 $\pm$ 2.4 & 61.2 $\pm$ 8.1 & 6.6 $\pm$ 2.8 & 1 / 6 / 3 & 1 / 0 / 5 & 10.5 $\pm$ 11.1 \\
Drift features inside the risk score & 0.0959 $\pm$ 0.0049 & 34.1 $\pm$ 29.8 & 63.5 $\pm$ 12.0 & 7.7 $\pm$ 4.5 & 5 / 9 / 5 & 1 / 2 / 6 & 44.6 $\pm$ 37.9 \\
Drift features only & 0.1989 $\pm$ 0.0034 & 0.0 & 0.0 & 0.0 & 0 / 0 / 3 & 0 / 0 / 0 & 100.0 \\
\textbf{B5: \sj{} risk, ID LTT $\tau$ + routing} & 0.0948 $\pm$ 0.0063 & 1.4 $\pm$ 2.5 & 60.1 $\pm$ 7.4 & 6.2 $\pm$ 2.8 & 0 / 6 / 3 & 1 / 0 / 5 & 41.3 $\pm$ 29.3 \\
\bottomrule\end{tabular}}
\end{table}

\begin{table}[!htbp]\centering\small
\caption{Isolating the threshold rule from the risk score (both Qwen backbones, 3 seeds, 27 tests per $\alpha$). All thresholds are re-certified offline from the stored calibration splits of each run. Routing changes a threshold only when $\tau_{\mathrm{proxy}}<\tau_{\mathrm{ID}}$, which never happened at $\alpha=0.05$; ``$\tau_{\mathrm{shift}}$ for all inputs'' applies the stricter of the ID and proxy thresholds everywhere.}\label{tab:routing}
\resizebox{\textwidth}{!}{%
\begin{tabular}{lllrrrrr}\toprule
Model & Risk score & Threshold rule & Shift cov. .05/.10/.20 (\%) & Risk@.10 (\%) & ID cov.@.10 (\%) & Violations .05/.10/.20 & Sum \\\midrule
0.6B & uncertainty & pooled (ID + proxy) & 20.2 / 37.9 / 66.9 & 6.4 & 58.8 & 4 / 5 / 13 & 22 \\
 &  & ID only & 0.0 / 35.4 / 60.7 & 5.3 & 54.6 & 0 / 4 / 8 & 12 \\
 &  & ID + routing to $\tau_{\mathrm{shift}}$ (B5 rule) & 0.0 / 35.4 / 60.7 & 5.3 & 54.6 & 0 / 4 / 8 & 12 \\
 &  & $\tau_{\mathrm{shift}}$ for all inputs & 0.0 / 35.0 / 60.3 & 5.2 & 53.9 & 0 / 3 / 8 & 11 \\
 & \sj{} & pooled (ID + proxy) & 18.0 / 39.1 / 63.3 & 7.5 & 60.8 & 4 / 8 / 10 & 22 \\
 &  & ID only & 0.0 / 33.3 / 68.0 & 5.5 & 51.1 & 0 / 4 / 14 & 18 \\
 &  & ID + routing to $\tau_{\mathrm{shift}}$ (B5 rule) & 0.0 / 33.3 / 67.2 & 5.2 & 51.1 & 0 / 4 / 12 & 16 \\
 &  & $\tau_{\mathrm{shift}}$ for all inputs & 0.0 / 32.8 / 56.8 & 5.8 & 50.1 & 0 / 3 / 4 & 7 \\
4B & uncertainty & pooled (ID + proxy) & 50.5 / 55.8 / 69.0 & 4.8 & 89.1 & 7 / 3 / 4 & 14 \\
 &  & ID only & 1.4 / 61.2 / 67.3 & 6.6 & 93.9 & 1 / 6 / 3 & 10 \\
 &  & ID + routing to $\tau_{\mathrm{shift}}$ (B5 rule) & 1.4 / 60.9 / 66.9 & 6.7 & 93.9 & 1 / 6 / 3 & 10 \\
 &  & $\tau_{\mathrm{shift}}$ for all inputs & 1.4 / 52.0 / 60.6 & 3.6 & 84.3 & 1 / 2 / 2 & 5 \\
 & \sj{} & pooled (ID + proxy) & 50.1 / 55.9 / 68.4 & 4.7 & 89.2 & 6 / 3 / 3 & 12 \\
 &  & ID only & 1.4 / 60.4 / 67.2 & 6.2 & 92.9 & 0 / 6 / 3 & 9 \\
 &  & ID + routing to $\tau_{\mathrm{shift}}$ (B5 rule) & 1.4 / 60.1 / 66.9 & 6.2 & 92.9 & 0 / 6 / 3 & 9 \\
 &  & $\tau_{\mathrm{shift}}$ for all inputs & 1.4 / 52.7 / 60.7 & 3.8 & 85.3 & 0 / 2 / 2 & 4 \\
\bottomrule\end{tabular}}
\end{table}

\begin{figure}[!htbp]
  \centering
  \includegraphics[width=\textwidth]{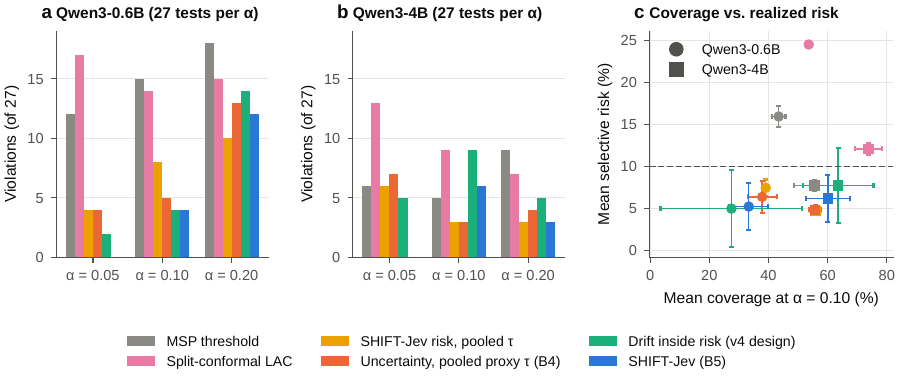}
  \caption{(a, b)~Selective-risk violations out of 27 (environment, seed) tests at each target risk level. At $\alpha=0.05$, B5 abstains on (almost) all inputs (coverage 0.0\% and 1.4\%), so its zero count there is vacuous; see \cref{tab:abst_small,tab:abst_large}. (c)~Mean coverage against mean realized risk at $\alpha=0.10$ (error bars: s.d.\ over seeds); the dashed line marks the target.}
  \label{fig:riskcontrol}
\end{figure}

\begin{figure}[!htbp]
  \centering
  \includegraphics[width=\textwidth]{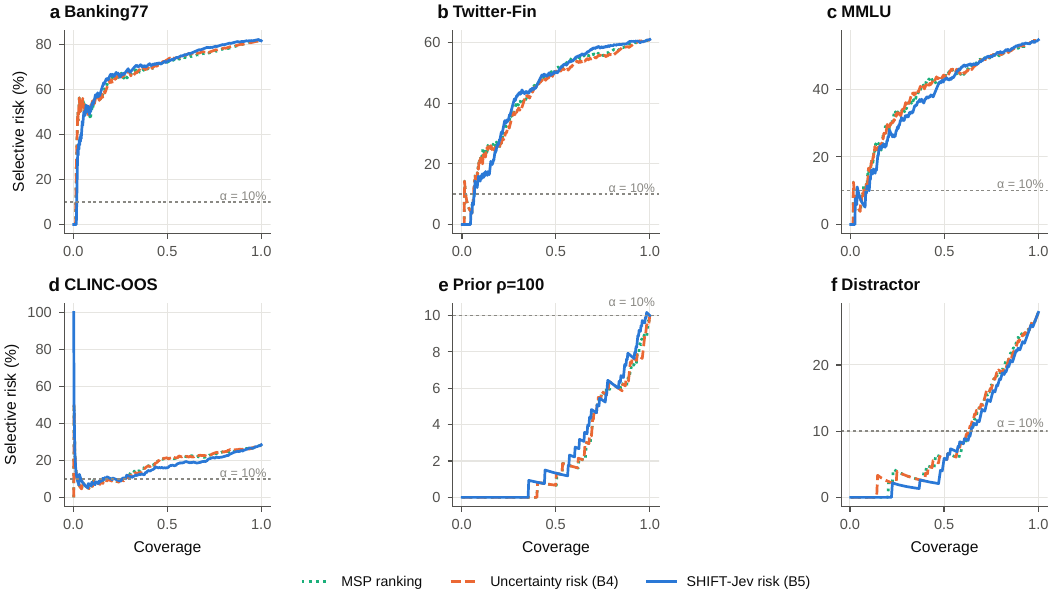}
  \caption{Risk--coverage curves on six shifted environments (Qwen3-0.6B, seed 1) for MSP, the uncertainty risk score (B4) and the \sj{} risk score (B5). The dotted line marks $\alpha=10\%$.}
  \label{fig:riskcov}
\end{figure}

\begin{table}[!htbp]\centering\small
\caption{Candidate-set perturbations (seed 1). For gold-removed items every answer is wrong, so abstention should be high. Rotation agreement is computed only at 0.6B (the consistency signal is not used at 4B).}\label{tab:pert}
\resizebox{\textwidth}{!}{%
\begin{tabular}{lllrrrrrr}\toprule
Model & Perturbation & Type & Acc B1 (\%) & Acc B2 (\%) & Rotation agreement (\%) & Routed to $\tau_{\mathrm{shift}}$ (\%) & B4 abstain / risk (\%) & B5 abstain / risk (\%) \\\midrule
0.6B & Option rotation & harmless & 70.6 & 77.5 & 87.7 & 0 & 38.7 / 9.6 & 49.5 / 8.7 \\
0.6B & Label paraphrase & harmless & 70.6 & 77.1 & 88.9 & 0 & 35.9 / 10.2 & 47.7 / 8.8 \\
0.6B & Near-duplicate distractor & harmless & 59.3 & 72.1 & 77.0 & 0 & 47.1 / 6.5 & 61.8 / 2.6 \\
0.6B & Gold option removed & task-altering & -- & -- & 81.9 & 0 & 32.8 / 100.0 & 42.2 / 100.0 \\
4B & Option rotation & harmless & 92.6 & 90.2 & -- & 0 & 9.3 / 5.9 & 10.8 / 4.4 \\
4B & Label paraphrase & harmless & 89.5 & 90.2 & -- & 0 & 9.2 / 7.2 & 5.9 / 6.2 \\
4B & Near-duplicate distractor & harmless & 89.7 & 89.2 & -- & 0 & 19.6 / 3.7 & 15.7 / 3.5 \\
4B & Gold option removed & task-altering & -- & -- & -- & 0 & 14.2 / 100.0 & 72.1 / 100.0 \\
\bottomrule\end{tabular}}
\end{table}

\subsection{Shift detection in decision streams}
\label{sec:streams}
Domain shift and emerging classes are easy to detect (\cref{tab:detect}, \cref{fig:detect}): CUSUM on the drift distance or on the \sj{} risk score detects 97--100\% of full-severity streams within 2--13 items. Class-prior shift is invisible to all confidence- and embedding-based signals (0--1\% detection), because each input still looks like an ID input and the model remains equally confident. Only the predicted-label $\chi^2$ test detects it, in 95\% (Qwen3-0.6B) and 99\% (Qwen3-4B) of full-severity streams and 80\% and 86\% at 50\% severity, after about 36 items. It is also the best detector of candidate perturbations (70\% and 64\%). On fresh no-shift streams its false-alarm rate is 6\% and 10\%, above the 5\% calibration target on both backbones; CUSUM on drift and on the \sj{} risk (6\% and 7\%, Qwen3-0.6B) and the KS test (7\%, Qwen3-4B) also exceed it, so the simulated calibration of alarm levels is somewhat optimistic. A drift CUSUM and the label test together cover all four shift types. Stream results use the seed-1 adapter.

\begin{table}[!htbp]\centering\small
\caption{Stream-level shift detection: detection rate (\%) at severity 100\% (mean delay in items) $\cdot$ at 50\%, 50 streams per environment and severity. Every alarm threshold is calibrated to $\le$5\% false alarms on 200 simulated no-shift streams; the last column is the observed false-alarm rate on 50 fresh no-shift streams per pool.}\label{tab:detect}
\resizebox{\textwidth}{!}{%
\begin{tabular}{llrrrrr}\toprule
Model & Detector & Domain shift & Emerging classes & Class-prior shift & Candidate perturbation & False alarm (\%) \\\midrule
0.6B & CUSUM $\cdot$ max-probability & 98 (6) $\cdot$ 99 & 98 (11) $\cdot$ 100 & 0 $\cdot$ 0 & 34 (81) $\cdot$ 18 & 2 \\
0.6B & CUSUM $\cdot$ entropy & 33 (19) $\cdot$ 25 & 0 $\cdot$ 0 & 0 $\cdot$ 0 & 9 (88) $\cdot$ 5 & 4 \\
0.6B & CUSUM $\cdot$ $k$NN drift & 97 (12) $\cdot$ 98 & 98 (7) $\cdot$ 98 & 0 $\cdot$ 1 & 1 (90) $\cdot$ 1 & 6 \\
0.6B & CUSUM $\cdot$ \sj{} risk & 99 (6) $\cdot$ 98 & 100 (7) $\cdot$ 98 & 0 $\cdot$ 0 & 36 (57) $\cdot$ 17 & 7 \\
0.6B & KS window $\cdot$ \sj{} risk & 97 (13) $\cdot$ 98 & 100 (16) $\cdot$ 100 & 0 $\cdot$ 0 & 31 (90) $\cdot$ 15 & 3 \\
0.6B & $\chi^2$ $\cdot$ predicted labels (BBSE) & 97 (10) $\cdot$ 96 & 94 (10) $\cdot$ 96 & 95 (37) $\cdot$ 80 & 70 (96) $\cdot$ 27 & 6 \\
4B & CUSUM $\cdot$ max-probability & 97 (22) $\cdot$ 83 & 94 (70) $\cdot$ 34 & 0 $\cdot$ 0 & 8 (88) $\cdot$ 4 & 5 \\
4B & CUSUM $\cdot$ entropy & 6 (103) $\cdot$ 1 & 0 $\cdot$ 0 & 0 $\cdot$ 0 & 5 (110) $\cdot$ 1 & 4 \\
4B & CUSUM $\cdot$ $k$NN drift & 98 (11) $\cdot$ 99 & 98 (6) $\cdot$ 100 & 0 $\cdot$ 0 & 0 $\cdot$ 1 & 5 \\
4B & CUSUM $\cdot$ \sj{} risk & 97 (13) $\cdot$ 91 & 98 (2) $\cdot$ 96 & 0 $\cdot$ 0 & 20 (99) $\cdot$ 7 & 3 \\
4B & KS window $\cdot$ \sj{} risk & 99 (16) $\cdot$ 86 & 96 (10) $\cdot$ 98 & 0 $\cdot$ 0 & 33 (127) $\cdot$ 13 & 7 \\
4B & $\chi^2$ $\cdot$ predicted labels (BBSE) & 94 (10) $\cdot$ 99 & 98 (10) $\cdot$ 100 & 99 (36) $\cdot$ 86 & 64 (89) $\cdot$ 27 & 10 \\
\bottomrule\end{tabular}}
\end{table}

\begin{figure}[!htbp]
  \centering
  \includegraphics[width=\textwidth]{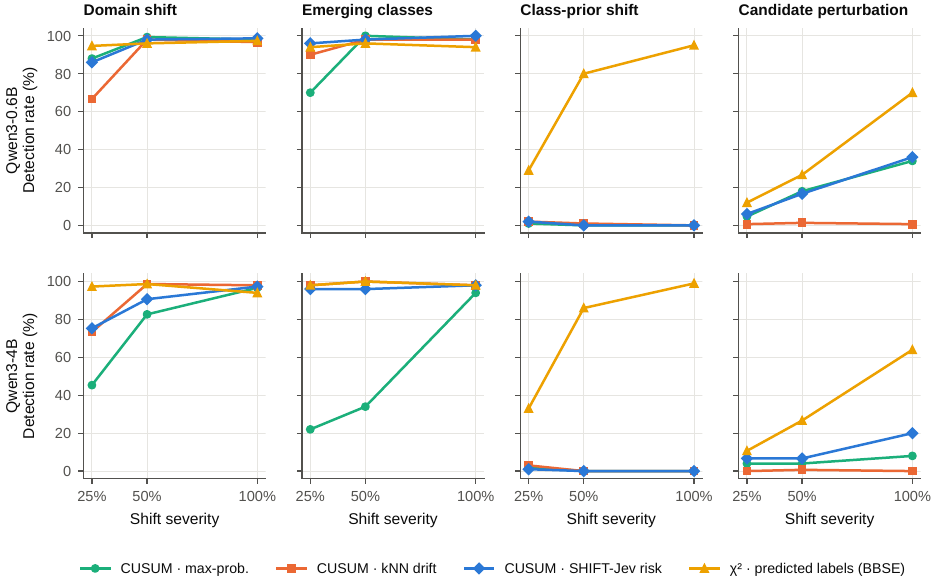}
  \caption{Detection rate against shift severity for four detectors (top: Qwen3-0.6B; bottom: Qwen3-4B; 50 streams per point). Only the predicted-label $\chi^2$ test detects class-prior shift.}
  \label{fig:detect}
\end{figure}

Switching to $\tshift$ after an alarm reduces risk violations in most streams (\cref{tab:stream}). On MMLU, where errors are confident, the share of streams violating $\alpha=0.10$ after the change point falls from 98\% to 44\% on Qwen3-0.6B and from 100\% to 84\% on Qwen3-4B. On AG News with imbalance ratio 100 it falls from 20\% to 0\% on Qwen3-4B, and on the near-duplicate distractor from 10\% to 0\% on Qwen3-0.6B. CLINC150 on Qwen3-0.6B is the exception (20\% to 38\%): the switch lowers coverage from 14\% to 5\% without making the remaining answers safer. For severe domain shifts both systems abstain on almost everything, which is safe but leaves those inputs unautomated.

\begin{table}[!htbp]\centering\small
\caption{Selective risk after a change point (50 streams per environment, $\alpha=0.10$, seed 1). \sj{} switches to $\tau_{\mathrm{shift}}$ after a drift-CUSUM or label-$\chi^2$ alarm.}\label{tab:stream}
\resizebox{\textwidth}{!}{%
\begin{tabular}{lrrrrrr}\toprule
Environment & 0.6B B4 cov / risk & 0.6B \sj{} cov / risk & 0.6B violating streams B4 / \sj{} (\%) & 4B B4 cov / risk & 4B \sj{} cov / risk & 4B violating streams B4 / \sj{} (\%) \\\midrule
Banking77 & 0.0 / 0.0 & 0.0 / 0.0 & 0 / 0 & 0.0 / 0.0 & 0.0 / 0.0 & 0 / 0 \\
Twitter-Fin & 0.0 / 0.0 & 0.0 / 0.0 & 0 / 0 & 0.0 / 0.0 & 0.0 / 0.0 & 0 / 0 \\
MMLU & 18.2 / 27.3 & 9.2 / 9.9 & 98 / 44 & 66.5 / 16.0 & 59.3 / 12.5 & 100 / 84 \\
AG prior $\rho$=10 & 87.3 / 5.4 & 75.8 / 4.1 & 0 / 0 & 93.2 / 6.1 & 88.7 / 4.5 & 4 / 0 \\
AG prior $\rho$=100 & 89.0 / 7.1 & 78.0 / 6.5 & 6 / 4 & 93.5 / 8.6 & 86.3 / 5.5 & 20 / 0 \\
CLINC-OOS & 14.0 / 7.2 & 4.6 / 8.6 & 20 / 38 & 0.0 / 0.0 & 0.0 / 0.0 & 0 / 0 \\
Option rotation & 61.4 / 9.7 & 51.2 / 9.1 & 46 / 40 & 90.6 / 5.6 & 89.0 / 4.4 & 2 / 0 \\
Label paraphrase & 63.4 / 10.3 & 51.4 / 9.0 & 60 / 36 & 90.5 / 7.2 & 90.0 / 6.1 & 8 / 2 \\
Near-duplicate distractor & 53.4 / 6.5 & 38.1 / 2.5 & 10 / 0 & 80.3 / 3.8 & 78.6 / 3.1 & 0 / 0 \\
\bottomrule\end{tabular}}
\end{table}

\subsection{RQ3: continual adaptation}
\label{sec:rq3}
On Qwen3-0.6B \sj{} has the highest final average accuracy (\cref{tab:cl}, \cref{fig:cl}): $70.48\pm0.88\%$, against 68.10\% for LwF, 66.65\% for naive fine-tuning, 66.67\% for experience replay, 66.48\% for EWC and 66.33\% for the base paper's Seq-FT+KL recipe. It is better than every baseline in all three seeds, by $2.38\pm2.52$ points (LwF) to $4.15\pm2.74$ points (Seq-FT+KL). On the 1,600 unique evaluation items, with each item's correctness averaged over the three seeds, the differences are significant (sign test $p=8.5\times10^{-4}$ against LwF and $p\le6.7\times10^{-7}$ against the others). With only three seeds, a seed-level paired test is not significant ($p=0.07$--0.24), so the evidence rests on the consistent direction across seeds and on the item-level differences. \sj{} also has the most negative forgetting ($-2.1\pm1.9$ points, meaning that earlier regimes improve after later stages) and the largest backward transfer ($+3.0$). Most of the gain comes from Banking77 (54.7\% against 41.8--46.2\%); on the other regimes it is within 1.7 points of the best baseline.

The ablations show that most of this gain comes from the training recipe rather than from the gate (last columns of \cref{tab:cl}). Removing replay costs $4.81\pm2.22$ points, replacing the severity-scaled anchor weight with a fixed one $3.27\pm1.50$, random instead of margin-based replay $2.38\pm2.90$, and removing anchors and distillation $1.62\pm2.05$ points. Removing the gate costs only $1.23\pm2.18$ points and helps in two of three seeds (item-level $p=0.024$). The gate's main contribution is to stop updates that do not help: with it, forgetting is $-2.1$ points instead of $+1.6$, and the no-shift control stage is rejected in every seed. The choice of gate matters more than its presence: replacing the expected-gain gate with the earlier gate, which tests whether the model has degraded, costs $4.58\pm1.09$ points in all three seeds because that gate never adapts to CLINC150 (\cref{tab:gate}).

The gate makes the right decision in every seed on Qwen3-0.6B (\cref{tab:gate}, \cref{fig:gate}). It stops the control stage at the probe (held-out gains between $-3.8$ and 0 points) and accepts Banking77, Twitter-Fin and CLINC150 with held-out gains of 11--33 points ($p<0.025$). The earlier gate skips CLINC150 in every seed because the model is not \emph{worse} there, although it can become much better. The gate also responds to the amount of data: with 75 labelled Banking77 items it rejects the update, while with 150 and 300 items it accepts updates that raise Banking77 accuracy from 18.2\% to 39.8\% and 46.8\% at unchanged ID accuracy (\cref{tab:sweeps}).

\begin{table}[!htbp]\centering\small
\caption{Continual adaptation on Qwen3-0.6B. Stream: Banking77 $\rightarrow$ control (no shift) $\rightarrow$ Twitter-Fin $\rightarrow$ CLINC-OOS, starting from B2. Final accuracy (\%) on 400 held-out items per regime; mean $\pm$ s.d.\ over 3 seeds. The last three columns compare \sj{} with each row: the per-seed paired difference in average accuracy (mean $\pm$ s.d.), the number of seeds in which \sj{} is better, and a two-sided sign test on the 1,600 unique evaluation items using each item's correctness averaged over the three seeds (the evaluation items are shared across seeds, so seeds are not pooled as independent pairs).}\label{tab:cl}
\resizebox{\textwidth}{!}{%
\begin{tabular}{lrrrrrrrrrrrr}\toprule
Method (3 seeds) & Avg acc (\%) & BWT (pts) & Forgetting (pts) & ID & Banking77 & Twitter-Fin & CLINC & Stages adapted & Train (min) & \sj{} $-$ row (pts) & Seeds \sj{} better & $p$ (items) \\\midrule
Naive fine-tuning & 66.65 $\pm$ 1.17 & $-$1.1 $\pm$ 1.0 & 1.1 $\pm$ 1.0 & 76.9 $\pm$ 0.9 & 42.7 $\pm$ 2.9 & 60.5 $\pm$ 2.2 & 86.5 $\pm$ 2.0 & 4.0/4 & 9.3 & +3.83 $\pm$ 1.88 & 3/3 & $1.1\times10^{-9}$ \\
EWC & 66.48 $\pm$ 2.10 & 0.4 $\pm$ 1.6 & $-$0.3 $\pm$ 1.7 & 77.2 $\pm$ 0.8 & 42.0 $\pm$ 4.6 & 59.9 $\pm$ 3.5 & 86.8 $\pm$ 1.6 & 4.0/4 & 9.6 & +4.00 $\pm$ 2.88 & 3/3 & $6.3\times10^{-8}$ \\
Seq-FT + KL anchors (base paper) & 66.33 $\pm$ 1.92 & 0.3 $\pm$ 1.0 & $-$0.2 $\pm$ 1.1 & 77.1 $\pm$ 0.9 & 41.8 $\pm$ 4.1 & 59.4 $\pm$ 2.8 & 87.0 $\pm$ 1.3 & 4.0/4 & 9.5 & +4.15 $\pm$ 2.74 & 3/3 & $1.2\times10^{-8}$ \\
Experience replay & 66.67 $\pm$ 1.71 & $-$0.3 $\pm$ 2.7 & 1.4 $\pm$ 2.5 & 76.2 $\pm$ 1.1 & 46.2 $\pm$ 3.2 & 56.5 $\pm$ 5.7 & 87.8 $\pm$ 1.0 & 4.0/4 & 13.4 & +3.81 $\pm$ 2.53 & 3/3 & $6.7\times10^{-7}$ \\
LwF (tree-local) & 68.10 $\pm$ 1.71 & 0.3 $\pm$ 1.3 & $-$0.3 $\pm$ 1.3 & 77.2 $\pm$ 1.3 & 43.6 $\pm$ 2.5 & 63.0 $\pm$ 2.4 & 88.6 $\pm$ 2.7 & 4.0/4 & 9.4 & +2.38 $\pm$ 2.52 & 3/3 & $8.5\times10^{-4}$ \\
\textbf{\sj{} (expected-gain gate)} & 70.48 $\pm$ 0.88 & 3.0 $\pm$ 0.9 & $-$2.1 $\pm$ 1.9 & 77.6 $\pm$ 0.9 & 54.7 $\pm$ 1.9 & 61.3 $\pm$ 1.9 & 88.3 $\pm$ 0.6 & 3.0/4 & 12.4 & -- & -- & -- \\
\midrule\multicolumn{13}{l}{\emph{Ablations of \sj{} (3 seeds)}} \\
\quad $-$ gate \& rollback (always adapt) & 69.25 $\pm$ 1.32 & 0.7 $\pm$ 2.0 & 1.6 $\pm$ 1.2 & 76.4 $\pm$ 0.6 & 54.7 $\pm$ 0.7 & 58.0 $\pm$ 3.5 & 87.9 $\pm$ 0.8 & 4.0/4 & 14.0 & +1.23 $\pm$ 2.18 & 2/3 & 0.024 \\
\quad v4 shift gate instead of gain gate & 65.90 $\pm$ 0.34 & 0.2 $\pm$ 0.5 & 0.3 $\pm$ 0.1 & 78.4 $\pm$ 0.4 & 49.4 $\pm$ 2.5 & 57.6 $\pm$ 2.1 & 78.2 $\pm$ 1.8 & 2.0/4 & 9.3 & +4.58 $\pm$ 1.09 & 3/3 & $1.5\times10^{-9}$ \\
\quad $-$ replay & 65.67 $\pm$ 2.30 & $-$0.6 $\pm$ 0.8 & 0.6 $\pm$ 0.8 & 78.0 $\pm$ 1.6 & 42.8 $\pm$ 3.8 & 54.9 $\pm$ 8.0 & 87.0 $\pm$ 0.4 & 2.7/4 & 10.3 & +4.81 $\pm$ 2.22 & 3/3 & $6.1\times10^{-17}$ \\
\quad $-$ KL anchors \& distillation & 68.85 $\pm$ 1.63 & 0.3 $\pm$ 1.2 & 1.9 $\pm$ 1.6 & 77.2 $\pm$ 0.9 & 51.9 $\pm$ 4.7 & 58.9 $\pm$ 5.7 & 87.3 $\pm$ 1.6 & 3.0/4 & 13.7 & +1.62 $\pm$ 2.05 & 2/3 & $1.7\times10^{-4}$ \\
\quad $-$ severity-scaled $\lambda$ & 67.21 $\pm$ 0.69 & 0.6 $\pm$ 1.4 & 1.0 $\pm$ 1.4 & 77.4 $\pm$ 1.2 & 46.3 $\pm$ 2.0 & 58.1 $\pm$ 1.8 & 87.0 & 3.0/4 & 13.4 & +3.27 $\pm$ 1.50 & 3/3 & $1.1\times10^{-10}$ \\
\quad $-$ risk-stratified replay (random) & 68.10 $\pm$ 2.30 & 0.9 $\pm$ 2.7 & $-$0.4 $\pm$ 2.0 & 77.0 $\pm$ 0.7 & 47.9 $\pm$ 8.1 & 58.5 $\pm$ 3.5 & 89.0 $\pm$ 0.7 & 3.0/4 & 12.9 & +2.38 $\pm$ 2.90 & 2/3 & $6.4\times10^{-5}$ \\
\bottomrule\end{tabular}}
\end{table}

\begin{figure}[!htbp]
  \centering
  \includegraphics[width=\textwidth]{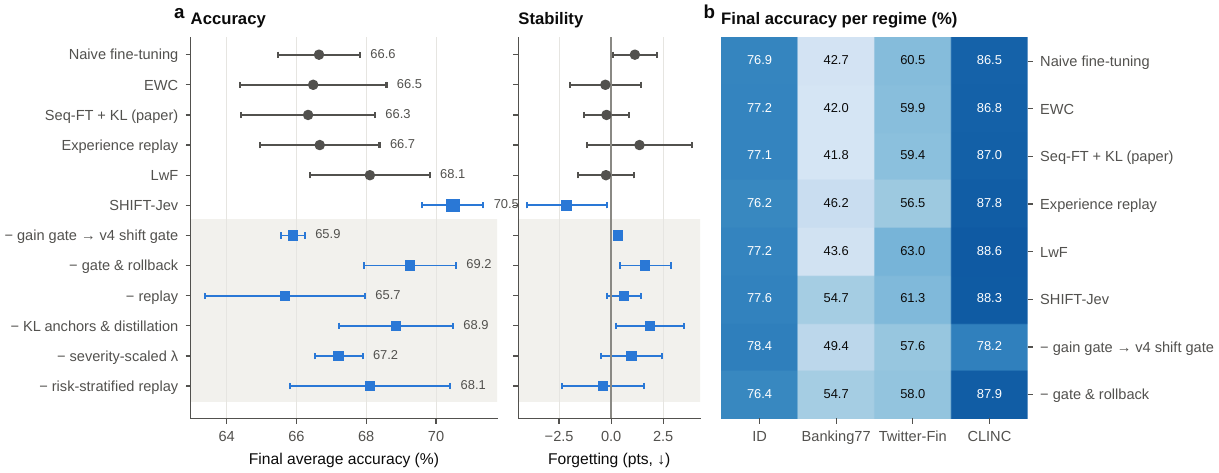}
  \caption{Continual adaptation on Qwen3-0.6B. (a)~Final average accuracy and forgetting (mean $\pm$ s.d., 3 seeds; shaded: \sj{} ablations). (b)~Final accuracy per regime: the gain over the baselines comes from Banking77, and the earlier gate loses CLINC150, the regime it never adapts to.}
  \label{fig:cl}
\end{figure}

\begin{table}[!htbp]\centering\small
\caption{Expected-gain gate decisions of \sj{} (3 seeds per backbone). Gain = held-out accuracy gain (points) of the best candidate on 80 new-regime items, with its one-sided sign-test $p$; an update is kept iff $p<0.10$ and old-regime accuracy drops by $\le$5 points. ``Trigger'' = the v4 shift test's outcome (logged only); the last column shows how often the v4 shift gate adapted.}\label{tab:gate}
\resizebox{\textwidth}{!}{%
\begin{tabular}{llrrrrlr}\toprule
Model & Regime & Accepted & Stopped by probe & Gain (pts) & Sign-test $p$ & v4 trigger & v4 gate adapted \\\midrule
Qwen3-0.6B & Banking77 & 3/3 & 0/3 & 18.8 to 32.5 & $1.3\times10^{-6}$ to $2.0\times10^{-3}$ & drift+perf, perf & 3/3 \\
 & Control (no shift) & 0/3 & 3/3 & $-$3.8 to 0.0 & 0.66 to 0.97 & none & 0/3 \\
 & Twitter-Fin & 3/3 & 0/3 & 11.2 to 15.0 & $2.1\times10^{-3}$ to 0.025 & perf & 3/3 \\
 & CLINC-OOS & 3/3 & 0/3 & 17.5 to 26.2 & $4.8\times10^{-7}$ to $6.6\times10^{-4}$ & none & 0/3 \\
Qwen3-4B & Banking77 & 3/3 & 0/3 & 10.0 to 12.5 & $3.2\times10^{-3}$ to 0.029 & perf & -- \\
 & Control (no shift) & 0/3 & 2/3 & $-$2.5 to 1.2 & 0.5 to 1 & none & -- \\
 & Twitter-Fin & 0/3 & 0/3 & 2.5 to 7.5 & 0.11 to 0.4 & perf & -- \\
 & CLINC-OOS & 3/3 & 0/3 & 7.5 to 10.0 & $3.9\times10^{-3}$ to 0.016 & none & -- \\
\bottomrule\end{tabular}}
\end{table}

\begin{figure}[!htbp]
  \centering
  \includegraphics[width=\textwidth]{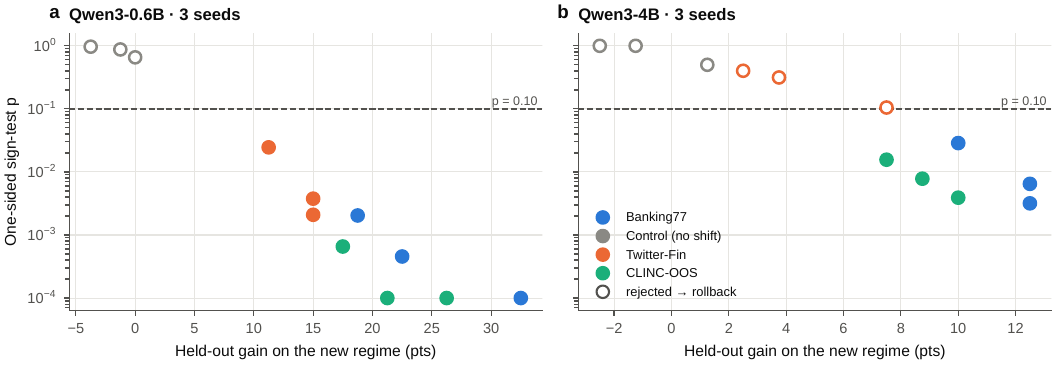}
  \caption{All expected-gain gate decisions: held-out gain of the best candidate against its one-sided sign-test $p$-value (log scale). Filled markers: accepted; open markers: rolled back. On Qwen3-4B the Twitter-Fin gains of 2.5--7.5 points do not reach $p<0.10$ on 80 items.}
  \label{fig:gate}
\end{figure}

On Qwen3-4B the gate is too conservative for small gains (\cref{tab:cl4b}, \cref{fig:cl4b}). \sj{} again rejects the control stage and accepts Banking77 and CLINC150 in all seeds, giving the best Banking77 (68.6\%) and CLINC150 (96.0\%) accuracy, the largest backward transfer ($+1.9$) and the least forgetting ($-1.8$). It rejects the Twitter-Fin update in all three seeds, however, because gains of 2.5--7.5 points on 80 held-out items do not reach $p<0.10$. Baselines that adapted reach 68.7--73.0\% on Twitter-Fin against 65.9\%, so \sj{} ends at $79.86\pm1.10\%$ average accuracy, 0.42--1.36 points below the baselines and below four of them in all three seeds. On the 1,200 unique evaluation items none of these differences is significant ($p=0.055$--0.94). A strong model leaves only small gains, and an 80-item paired test lacks the power to confirm them.

\begin{table}[!htbp]\centering\small
\caption{Continual adaptation on Qwen3-4B-Instruct-2507 (3 seeds; 300 held-out items per regime). The last three columns are computed as in \cref{tab:cl}, with 1,200 unique evaluation items; negative differences mean that \sj{} is worse.}\label{tab:cl4b}
\resizebox{\textwidth}{!}{%
\begin{tabular}{lrrrrrrrrrrrr}\toprule
Method (3 seeds) & Avg acc (\%) & BWT (pts) & Forgetting (pts) & ID & Banking77 & Twitter-Fin & CLINC & Stages adapted & Train (min) & \sj{} $-$ row (pts) & Seeds \sj{} better & $p$ (items) \\\midrule
Naive fine-tuning & 81.22 $\pm$ 0.55 & 0.8 $\pm$ 1.4 & $-$0.4 $\pm$ 1.4 & 89.2 $\pm$ 1.1 & 66.9 $\pm$ 1.8 & 73.0 $\pm$ 1.0 & 95.8 $\pm$ 0.5 & 4.0/4 & 15.0 & $-$1.36 $\pm$ 0.59 & 0/3 & 0.31 \\
EWC & 81.00 $\pm$ 1.01 & 1.0 $\pm$ 1.2 & $-$0.5 $\pm$ 0.9 & 89.3 $\pm$ 1.2 & 67.0 $\pm$ 1.2 & 72.0 $\pm$ 2.0 & 95.7 $\pm$ 0.3 & 4.0/4 & 15.2 & $-$1.14 $\pm$ 0.10 & 0/3 & 0.055 \\
Seq-FT + KL anchors (base paper) & 80.28 $\pm$ 1.12 & 0.7 $\pm$ 1.3 & $-$0.3 $\pm$ 1.0 & 89.8 $\pm$ 1.6 & 64.7 $\pm$ 2.2 & 71.3 $\pm$ 1.5 & 95.3 $\pm$ 0.3 & 4.0/4 & 15.1 & $-$0.42 $\pm$ 0.22 & 0/3 & 0.94 \\
Experience replay & 80.33 $\pm$ 0.76 & $-$0.4 $\pm$ 1.1 & 1.6 $\pm$ 0.6 & 89.9 $\pm$ 0.8 & 67.7 $\pm$ 1.2 & 68.7 $\pm$ 1.8 & 95.1 $\pm$ 1.3 & 4.0/4 & 18.6 & $-$0.47 $\pm$ 1.43 & 1/3 & 0.94 \\
LwF (tree-local) & 81.00 $\pm$ 0.60 & 1.0 $\pm$ 0.7 & $-$0.6 $\pm$ 0.7 & 89.7 $\pm$ 1.2 & 66.0 $\pm$ 0.9 & 72.9 $\pm$ 1.0 & 95.4 $\pm$ 0.8 & 4.0/4 & 15.0 & $-$1.14 $\pm$ 0.50 & 0/3 & 0.14 \\
\textbf{\sj{} (expected-gain gate)} & 79.86 $\pm$ 1.10 & 1.9 $\pm$ 1.0 & $-$1.8 $\pm$ 1.0 & 89.0 $\pm$ 0.9 & 68.6 $\pm$ 1.0 & 65.9 $\pm$ 3.4 & 96.0 $\pm$ 0.3 & 2.0/4 & 17.9 & -- & -- & -- \\
\bottomrule\end{tabular}}
\end{table}

\begin{figure}[!htbp]
  \centering
  \includegraphics[width=\textwidth]{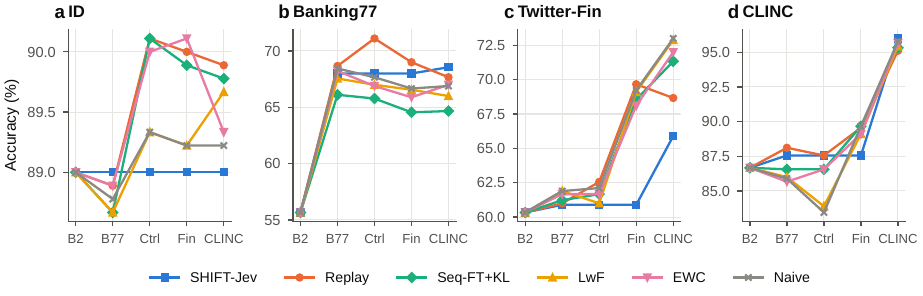}
  \caption{Qwen3-4B-Instruct: accuracy on each regime after every stage (mean of 3 seeds). \sj{} keeps ID and Banking77 stable and is best on CLINC150, but stays at its starting level on Twitter-Fin after rejecting that update.}
  \label{fig:cl4b}
\end{figure}

Recertifying one threshold per adapted regime on 160 held-out items (B6, \cref{tab:post}) restores CLINC150 coverage at low risk (44.5\% at 4.5\% risk on Qwen3-0.6B and 67.0\% at 3.5\% on Qwen3-4B) and keeps the ID regime within $\alpha$ on Qwen3-4B (85.7\% coverage at 4.7\% risk). Banking77 is fully abstained on, Twitter-Fin is answered for only 1.2\% and 7.3\% of inputs at realized risks of 40\% and 27\%, and on Qwen3-0.6B the ID regime is violated (64.5\% coverage at 11.6\% risk). Certification on 160 items is therefore fragile, in line with the finding that coverage degrades faster than accuracy under sequential fine-tuning~\citep{shihab2026continualcal}.

\begin{table}[!htbp]\centering\small
\caption{Reliability after continual adaptation (B6, seed 1): coverage / realized risk (\%) at $\alpha=0.10$ on the continual evaluation sets after the last stage. Each adapted regime receives its own certified threshold.}\label{tab:post}
\resizebox{\textwidth}{!}{%
\begin{tabular}{lllrrrrr}\toprule
Model & Method & Eval set & Acc (\%) & ID-only $\tau$ & B4 pooled $\tau$ & B4 regime $\tau$ & \sj{} (B5) \\\midrule
0.6B & \sj & ID & 76.5 & 53.8 / 6.5 & 62.0 / 8.9 & 59.8 / 11.3 & 64.5 / 11.6 \\
 &  & Banking77 & 56.0 & 0.0 / 0.0 & 0.0 / 0.0 & 0.0 / 0.0 & 0.0 / 0.0 \\
 &  & Twitter-Fin & 60.2 & 0.0 / 0.0 & 0.0 / 0.0 & 1.2 / 40.0 & 1.2 / 40.0 \\
 &  & CLINC-OOS & 89.0 & 0.0 / 0.0 & 0.0 / 0.0 & 43.2 / 4.0 & 44.5 / 4.5 \\
 & Experience replay & ID & 77.0 & 45.2 / 5.0 & 62.7 / 9.6 & 51.5 / 9.7 & 57.5 / 10.0 \\
 &  & Banking77 & 45.0 & 0.0 / 0.0 & 0.0 / 0.0 & 0.0 / 0.0 & 0.0 / 0.0 \\
 &  & Twitter-Fin & 61.8 & 0.0 / 0.0 & 0.0 / 0.0 & 0.0 / 0.0 & 0.0 / 0.0 \\
 &  & CLINC-OOS & 86.8 & 0.0 / 0.0 & 0.0 / 0.0 & 42.5 / 4.1 & 44.5 / 5.1 \\
 & Seq-FT + KL anchors & ID & 76.2 & 54.5 / 7.3 & 58.5 / 8.1 & 50.0 / 5.5 & 46.5 / 5.4 \\
 &  & Banking77 & 43.2 & 0.0 / 0.0 & 0.0 / 0.0 & 0.0 / 0.0 & 0.0 / 0.0 \\
 &  & Twitter-Fin & 62.3 & 0.0 / 0.0 & 0.0 / 0.0 & 0.0 / 0.0 & 0.0 / 0.0 \\
 &  & CLINC-OOS & 87.5 & 0.0 / 0.0 & 0.0 / 0.0 & 0.0 / 0.0 & 0.0 / 0.0 \\
4B & \sj & ID & 89.3 & 91.0 / 5.9 & 90.7 / 5.9 & 91.3 / 5.8 & 85.7 / 4.7 \\
 &  & Banking77 & 68.0 & 0.0 / 0.0 & 0.0 / 0.0 & 0.0 / 0.0 & 0.0 / 0.0 \\
 &  & Twitter-Fin & 69.7 & 0.0 / 0.0 & 0.0 / 0.0 & 4.7 / 35.7 & 7.3 / 27.3 \\
 &  & CLINC-OOS & 96.3 & 0.0 / 0.0 & 0.0 / 0.0 & 68.0 / 3.4 & 67.0 / 3.5 \\
 & Experience replay & ID & 90.0 & 94.0 / 8.2 & 91.3 / 6.9 & 94.0 / 8.5 & 92.3 / 8.3 \\
 &  & Banking77 & 67.0 & 0.0 / 0.0 & 0.0 / 0.0 & 0.0 / 0.0 & 0.0 / 0.0 \\
 &  & Twitter-Fin & 68.0 & 0.0 / 0.0 & 0.0 / 0.0 & 8.7 / 19.2 & 9.3 / 21.4 \\
 &  & CLINC-OOS & 93.7 & 0.0 / 0.0 & 0.0 / 0.0 & 67.7 / 6.9 & 67.7 / 6.9 \\
 & Seq-FT + KL anchors & ID & 91.0 & 96.7 / 7.9 & 86.3 / 6.6 & 95.0 / 7.4 & 93.0 / 7.2 \\
 &  & Banking77 & 65.0 & 0.0 / 0.0 & 0.0 / 0.0 & 0.0 / 0.0 & 0.0 / 0.0 \\
 &  & Twitter-Fin & 73.0 & 0.0 / 0.0 & 0.0 / 0.0 & 5.3 / 12.5 & 7.0 / 23.8 \\
 &  & CLINC-OOS & 95.3 & 0.0 / 0.0 & 0.0 / 0.0 & 65.0 / 4.6 & 63.7 / 4.7 \\
\bottomrule\end{tabular}}
\end{table}

\subsection{Scale and model family}
\label{sec:scale}
Training-free Qwen3-4B-Instruct reaches a six-task external macro accuracy of 75.76, against 78.58 for the base paper's different Qwen3.5-4B checkpoint, and fine-tuning leaves it unchanged (75.71) while halving its ECE (0.220 to 0.105; \cref{tab:ext4b}, \cref{fig:scale}). For the frozen Phi-4-mini (one run, \cref{tab:phi}), the MSP score with a pooled threshold has no violations at $\alpha=0.10$ and four in total. The \sj{} layer has three in total but certifies no threshold at $\alpha=0.05$, has a higher realized risk at $\alpha=0.10$ than the uncertainty score (7.9\% against 5.2\%), a worse AURC (0.160 against 0.129) and lower abstention on removed-answer items (14\% against 25\%). Without fine-tuning and without the consistency signal, the format signal of a different model family adds noise to the risk model. On the frozen Qwen models the \sj{} layer certifies no threshold at $\alpha\le0.10$, whereas the MSP score with a pooled threshold answers 24.4\% (0.6B) and 52.6\% (4B) of shifted inputs at $\alpha=0.10$, with 7 and 0 violations. The learned risk models are therefore most useful on top of the fine-tuned and temperature-scaled model (B3), and a plain MSP score remains a reasonable choice for an unadapted backbone.

\begin{figure}[!htbp]
  \centering
  \includegraphics[width=\textwidth]{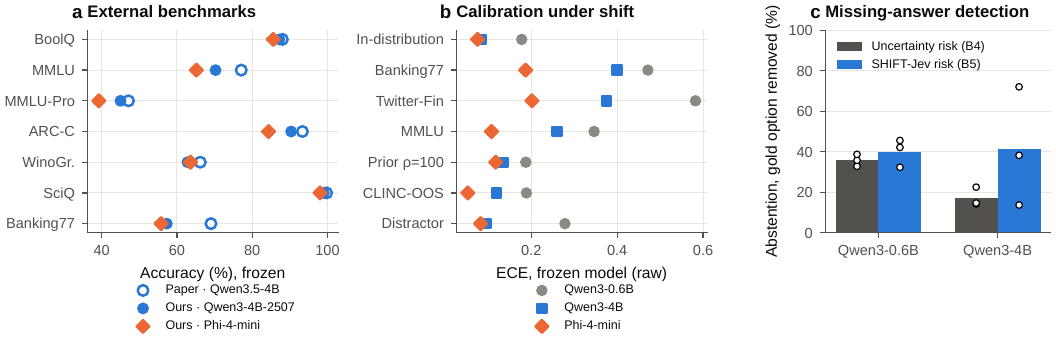}
  \caption{Scale and model family. (a)~Frozen accuracy on the base paper's external benchmarks. (b)~ECE of the frozen models under shift. (c)~Abstention when the correct option is removed (bars: mean of 3 seeds; dots: individual seeds).}
  \label{fig:scale}
\end{figure}

\subsection{Ablations of the base recipe and cost}
\label{sec:ablations}
The anchor ablation reproduces the base paper's finding (\cref{tab:anchor}, \cref{fig:ablations}). Decision accuracy hardly depends on the anchor weight (ID accuracy 79.8--80.2\%), but without anchors the KL divergence to the base model at the reply start rises from 0.0002 to 0.62 and 85 of 100 text-mode replies continue beyond the option identifier, against none with anchors. Banking77 accuracy falls from 81.2\% with 5 candidate intents to 18.6\% with all 77, while latency rises from 49 to 383~ms per item, and a rank-4 adapter on attention layers only (1.1M parameters) recovers about 60\% of the gain of the full rank-16 adapter (\cref{tab:sweeps}).

\begin{table}[!htbp]\centering\small
\caption{Base-paper anchor-weight ablation on Qwen3-0.6B (seed 1).}\label{tab:anchor}
\resizebox{\textwidth}{!}{%
\begin{tabular}{lrrrrrrrrrr}\toprule
Anchor weight & ID acc (\%) & ID ECE & Banking77 & MMLU & SciQ & KL start & KL answer & Text acc (\%) & Text $\equiv$ score (\%) & Runaway replies \\\midrule
$\lambda$ = 1 & 79.8 & 0.057 & 18.6 & 45.6 & 93.0 & 0.0002 & 0.0034 & 81 & 100 & 0/100 \\
$\lambda$ = 0.1 & 79.8 & 0.066 & 18.2 & 45.6 & 92.7 & 0.0008 & 0.0043 & 80 & 100 & 0/100 \\
$\lambda$ = 0 & 80.2 & 0.066 & 18.8 & 46.6 & 92.7 & 0.6204 & 0.2329 & 81 & 100 & 85/100 \\
\bottomrule\end{tabular}}
\end{table}

\begin{table}[!htbp]\centering\small
\caption{Sweeps on Qwen3-0.6B (seed 1). Left: candidate-set size on Banking77. Middle: labelled data for one \sj{} stage on Banking77 (the gate rejects the 75-item update). Right: adapter configuration (shift accuracy over nine environments; AURC of B5).}\label{tab:sweeps}
\resizebox{\textwidth}{!}{%
\begin{tabular}{rrrrr}\toprule $K$ & Acc (\%) & AURC B4 & AURC B5 & ms/item \\\midrule
5 & 81.2 & 0.056 & 0.052 & 49 \\
20 & 62.0 & 0.193 & 0.198 & 115 \\
77 & 18.6 & 0.695 & 0.693 & 383 \\
\bottomrule\end{tabular}\qquad\begin{tabular}{rlrr}\toprule Items & Gate & B77 (\%) & ID (\%) \\\midrule
75 & rollback & 18.2 & 79.0 \\
150 & full & 39.8 & 78.8 \\
300 & full & 46.8 & 78.5 \\
\bottomrule\end{tabular}\qquad\begin{tabular}{lrrrr}\toprule Config. & Params & ID & Shift & AURC \\\midrule
Frozen (B1) & 0 & 73.3 & 57.5 & 0.317 \\
r16, all linear & 10.1\,M & 79.7 & 64.6 & 0.224 \\
r4, attention & 1.1\,M & 77.1 & 62.1 & 0.241 \\
\bottomrule\end{tabular}}
\end{table}

\begin{figure}[!htbp]
  \centering
  \includegraphics[width=\textwidth]{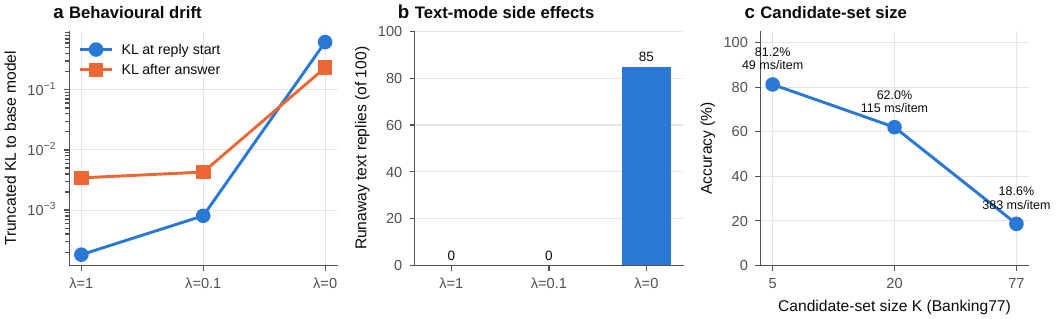}
  \caption{(a)~KL divergence to the base model at the reply start and after the answer, by anchor weight (log scale). (b)~Runaway text-mode replies out of 100. (c)~Accuracy and latency against candidate-set size on Banking77.}
  \label{fig:ablations}
\end{figure}

The reliability layer adds no forward passes except the rotation check, which doubles scoring cost and is therefore used only for the 0.6B model. Packed trie scoring takes 72~ms per item on Qwen3-0.6B and 92~ms on Qwen3-4B on a T4 GPU. Fine-tuning B2 takes 11.1 minutes on Qwen3-0.6B and 6.7 minutes on Qwen3-4B including reference caching, and a four-stage \sj{} run trains for 12.4 and 17.9 minutes, slightly less than experience replay (13.4 and 18.6 minutes). \Cref{tab:eff} gives the full cost breakdown.

\section{Discussion}
\label{sec:discussion}

\subsection{Implications for deploying decision models}
Under shift, both the risk score and the data used to certify its threshold determine whether selective risk stays within budget. All scores ranked errors almost identically, yet with the same pooled certification the MSP score violated $\alpha=0.10$ three times as often as a learned uncertainty score on Qwen3-0.6B, and split-conformal prediction sets, which guarantee marginal coverage rather than selective risk, violated it in a third to a half of the tests. This agrees with reports that LLM risk guarantees must be re-verified for each deployment~\citep{kotte2026crccertify} and enforced per group~\citep{salem2026hgcrc,rahman2026rolecrc}. No rule eliminated violations at a useful coverage: at $\alpha=0.10$ the best rules still violated 2--3 of 27 tests. Evaluations of decision models should therefore report realized selective risk and coverage under shift, not only calibration error, and should report coverage together with violation counts, since a rule that abstains on everything has no violations.

Drift tells the system which population an input belongs to, but in our experiments routing inputs to population-specific thresholds added little, because the proxy-shift threshold was rarely stricter than the ID threshold. A single conservative threshold, the stricter of those certified on ID and on proxy-shift data, gave the fewest violations; this post-hoc finding needs confirmation on further shifts. Drift is more useful for monitoring the stream, and label-shift monitoring should run by default: class-prior shift changes neither the inputs nor the model's confidence, so every confidence- and embedding-based detector missed it, while a $\chi^2$ test on predicted labels, which needs nothing beyond the decisions themselves, detected 95--99\% of shifted streams.

For continual adaptation, replay of boundary examples and severity-scaled anchoring to the previous model produced most of the accuracy gain on Qwen3-0.6B. The expected-gain gate added about one point on average, but it rejected every no-shift update, turned forgetting negative and, unlike a gate that tests for degradation, did not skip a regime on which the model could improve.

In practice, we recommend fine-tuning with KL anchors followed by temperature scaling, which lowers calibration error under shift two- to three-fold; certifying a learned uncertainty score with LTT on ID and proxy-shift data separately and applying the stricter threshold; running a predicted-label $\chi^2$ test next to a drift CUSUM; and gating updates on held-out gain with held-out and recertification sets larger than the 80 and 160 items used here when the backbone is strong.

\subsection{Limitations and threats to validity}
\label{sec:limits}
Several results are negative or weaker than the design intended. First, at $\alpha=0.05$ the ID-only certification of B5 failed in five of six Qwen runs and the model abstained on everything. A 400-item test half is too small to certify a 5\% risk unless the answered inputs are almost error-free, and the fixed-sequence order, which tests the strictest candidate first and stops at the first failure, discarded passing thresholds in one run. Larger calibration sets, or an ordering of candidates chosen on the first calibration half to maximize the chance of passing, would address this without affecting validity; we did not evaluate these changes. Second, routing by drift had little effect, as discussed above. Third, on Qwen3-4B the gate rejected a Twitter-Fin gain of 2.5--7.5 points that 80 paired items cannot confirm at $p<0.10$; larger held-out sets or a sequential test would recover such gains. Fourth, abstention on items whose correct option was removed varied between 14\% and 72\% across seeds on Qwen3-4B, so a dedicated ``none of the above'' mechanism may be needed. Fifth, for the frozen Phi-4-mini the additional signals increased realized risk, so feature-rich risk models should be validated for each backbone.

The statistical evidence is limited by three seeds per backbone and by evaluation items that are shared across seeds. We therefore avoided pooling seeds as independent pairs, but seed-level tests have little power, and violation counts over (environment, seed) tests are descriptive. The routing analysis in \cref{tab:routing} re-certifies thresholds offline from the stored calibration splits of the same runs and is post hoc.

The 4B comparisons with the base paper use a different checkpoint, and our fine-tuning mixture deliberately differs from the base paper's so that the test shifts remain unseen. Class-prior and candidate-set shifts are synthetic, although parameterized by severity, and JevBench~\citep{jevbench2026} was not available in our environment. The LTT and sign-test guarantees hold within a regime under exchangeability; for unseen shifts our violation counts are empirical. The present design was reached after earlier runs on the same benchmark exposed two failures, a gate that tested for degradation and drift features inside the risk score. All numbers reported here come from new runs with new seeds for the 4B model, new adapters and larger shift test sets, and the earlier designs are reported as ablations. Finally, the study covers two model families and models of up to 4B parameters on single-GPU budgets.

\section{Conclusion}
\label{sec:conclusion}
LLMs can serve as decision models, but deployment shifts their domain, label priors, label space and candidate sets. We evaluated such models on a benchmark of ten shifted environments and proposed \sj{}, which adds risk scoring with certified abstention, stream monitoring and gain-gated continual adaptation to the \jev{} interface. Fine-tuning with KL anchors and temperature scaling reduced calibration error under shift two- to three-fold. With the same certification, a learned uncertainty score kept selective risk within budget far more often than a max-probability score or conformal prediction sets, while drift-based routing added little and a single conservative threshold did best. A predicted-label $\chi^2$ test detected class-prior shifts that all other detectors missed. On Qwen3-0.6B, \sj{} was the most accurate continual learner in every seed, with negative forgetting and correct accept and reject decisions, and most of its gain came from replay and severity-scaled anchoring. On Qwen3-4B its gate was too conservative for the small remaining gains. Future work will repair the LTT ladder at strict risk levels, study sequential gain tests, larger per-regime calibration sets, label-shift-aware threshold correction~\citep{podkopaev2021labelshift}, more model families and naturally occurring shifts in production decision streams.

\section*{Data and Code Availability}
All datasets are public on the Hugging Face Hub. The implementation is a single Kaggle notebook that reproduces every table and figure in three 12-hour sessions on two T4 GPUs; it will be released with the camera-ready version.

\bibliographystyle{unsrtnat}
\bibliography{references}

\clearpage
\appendix
\section{Implementation Details}
\label{app:impl}
\paragraph{Answer interface and scoring.} We use each model's chat template with thinking disabled (Qwen3). The user message lists the state followed by ``\texttt{[1]} $o_1$ \dots\ \texttt{[K]} $o_K$'' and the instruction to answer with the bracketed identifier only, and the assistant turn is pre-filled with ``\texttt{Best answer: [}''. Identifier suffixes are taken from the tokenization of the full string, so multi-digit identifiers and merged bracket tokens are handled exactly, and the prefix-free property is checked for every item. For packed scoring, all trie nodes are appended after the prompt in one sequence; a token of a node attends to the whole prompt, to its ancestors and causally to itself, and position ids continue along each root-to-node path. When a model is loaded, packed and unpacked scores are compared on four items, and the packed path is used only if the maximum probability difference is below $10^{-4}$ (fp32) or 0.03 (fp16); the observed difference was at most $2\times10^{-5}$.

\paragraph{Fine-tuning objective.} The B2 loss is the sum of the cross-entropy over legal child tokens at each node of the correct option's trie path and three KL anchors with weight 1: on the probability mass of legal identifier tokens, on the distribution of illegal tokens, and on the full next-token distribution at the reply start and after the answer. Anchors are computed against the frozen model on correct-path and sampled off-path nodes with top-64 plus tail truncation.

\paragraph{Precision and hardware.} On T4 GPUs Qwen3-0.6B runs in fp32 because fp16 overflows in its attention, and the 4B and Phi models run in fp16 with fp32 LoRA weights. Each part of the study runs two lanes in parallel on the two T4 GPUs of one Kaggle session with a shared time budget, and results are merged afterwards.

\paragraph{Monitor and detectors.} The drift model standardizes the state embeddings of 400 ID validation items, projects them to 64 principal components and uses a cosine $k$-NN index with $k=10$; Mahalanobis~\citep{lee2018mahalanobis} and relative-Mahalanobis~\citep{ren2021rmd} distances with a Ledoit--Wolf covariance are computed for ablations. CUSUM uses a drift allowance of 0.5 on signals standardized by the reference pool, with the alarm level at the 95th percentile of no-shift maxima over 200 simulated streams. The Kolmogorov--Smirnov test uses windows of 25 items with step 5, and the $\chi^2$ test uses windows of 50 items over the predicted-category histogram with an ``unseen'' bin and add-0.5 smoothing.

\paragraph{Risk model and certification.} Features are standardized and the logistic regression uses $C=1$. Calibration sets are split 50/50 with a fixed seed. Candidate thresholds are the largest thresholds whose empirical selective risk on the first half is at most $c\alpha$ for $c\in\{0.2,0.35,0.5,0.65,0.8,0.9,1.0,1.1\}$. The Clopper--Pearson bound is the $1-\delta$ quantile of $\mathrm{Beta}(e+1,n-e)$ for $e$ errors among $n$ answered points. Split-conformal LAC uses the non-conformity score $1-P(y\mid x)$ on pooled calibration data and the $\lceil(n+1)(1-\alpha)\rceil/n$ empirical quantile.

\paragraph{Module C.} Shift measurement uses 250 incoming items and a reference pool of 200 ID validation items, extended by 100 items per adapted regime; a shift is \emph{validated} if either test gives $p<0.01$ with an effect of at least 5 points. The buffer stores 60 items per regime (half lowest-margin correct, half random) and replay adds 30\% of the new-data count. The gate uses 80 held-out new-regime items and the buffer's validation items (40 ID and 20 per adapted regime). EWC uses an online diagonal Fisher estimate over 100 examples with decay 0.9 and weight 50, and LwF uses tree-local distillation to the previous-stage model on new data with weight 1.

\section{Formal Guarantees and Module C Pseudocode}
\label{app:theory}

\begin{proposition}[Per-regime validity]
\label{prop:ltt}
If the calibration points of regime $g$ and a new test point from $g$ are exchangeable, and the risk model and the candidate thresholds depend only on the first calibration half, the fixed-sequence procedure of \cref{sec:moduleB} returns $\tau_g$ with $\Pr[R_{\mathcal D_g}(\tau_g)\le\alpha]\ge1-\delta$.
\end{proposition}
\begin{proof}[Proof sketch]
For each candidate threshold, the null hypothesis that its selective risk exceeds $\alpha$ is rejected only if the Clopper--Pearson bound, a valid level-$\delta$ test given the number of answered points, falls below $\alpha$. The order of the hypotheses is fixed before the second half is seen and testing stops at the first non-rejection, so the family-wise error rate is at most $\delta$ without multiplicity correction~\citep{angelopoulos2021ltt}. A violation requires a false rejection, which has probability at most $\delta$.
\end{proof}

\begin{proposition}[Controlled false acceptance]
\label{prop:gate}
If a candidate update is no more likely to be right where the current model is wrong than the reverse, the gate accepts it with probability at most $\gamma=0.10$; with two candidates, the probability of accepting any such update is at most $2\gamma$.
\end{proposition}
\begin{proof}[Proof sketch]
Given the number of discordant held-out pairs, the candidate's number of wins is stochastically no larger than a Binomial$(n,\tfrac12)$ variable under the null hypothesis, so the one-sided exact sign test has level $\gamma$. A union bound covers the two candidates, and the retention condition can only remove acceptances.
\end{proof}

\begin{algorithm}[H]
\caption{\sj{} Module C: one adaptation stage}
\label{alg:modulec}
\begin{algorithmic}[1]
\Require current adapter $\theta$, new batch $\mathcal S$ with held-out set $\mathcal V$, replay buffer $\mathcal B$ with validation set $\mathcal V_{\mathcal B}$
\State measure shift on $\mathcal S$; obtain \textit{validated} flag and severity $\hat s$
\State set anchor weight $\lambda_0e^{-\kappa\hat s}$ and distillation weight $\propto(1+\hat s)$
\State train on $\mathcal S$ + replay from $\mathcal B$; at 50\%: \textbf{if} not \textit{validated} and the checkpoint wins $\le$ loses on $\mathcal V$ \textbf{then} stop
\For{each candidate $\theta'\in\{\theta^{50\%},\theta^{100\%}\}$}
  \State accept $\theta'$ iff sign test vs.\ $\theta$ on $\mathcal V$ gives $p<0.10$ and accuracy on $\mathcal V_{\mathcal B}$ drops $\le5$ points
\EndFor
\State deploy the accepted candidate with the best mean of new- and old-regime accuracy, else keep $\theta$
\State add boundary and random items of $\mathcal S$ to $\mathcal B$; add the regime to the reference pools; certify its threshold
\end{algorithmic}
\end{algorithm}

\section{Additional Results}
\label{app:extra}
\begin{table}[!htbp]\centering\small
\caption{External benchmarks for the larger backbones (accuracy, \%; Banking77 = fixed random 1,000 of 3,080 items). The paper's 4B model is Qwen3.5-4B.}\label{tab:ext4b}
\resizebox{\textwidth}{!}{%
\begin{tabular}{lrrrrrrrrr}\toprule
System & BoolQ & MMLU & MMLU-Pro & ARC-C & WinoGr. & SciQ & Macro-6 & Banking77 & ECE (6) \\\midrule
Base paper, Qwen3.5-4B, training-free & 88.00 & 77.08 & 47.12 & 93.33 & 66.20 & 99.75 & 78.58 & 69.03 & -- \\
Base paper, Qwen3.5-4B, LoRA $\lambda$=1 & 88.20 & 77.00 & 46.90 & 92.70 & 68.80 & 100.00 & 78.93 & 75.40 & -- \\
Ours, Qwen3-4B-Instruct-2507, B1 & 87.20 & 70.25 & 45.00 & 90.33 & 62.80 & 99.00 & 75.76 & 57.30 & 0.220 \\
Ours, Qwen3-4B-Instruct-2507, B2 (seed 1) & 88.20 & 69.42 & 45.00 & 90.17 & 62.20 & 99.25 & 75.71 & 57.10 & 0.105 \\
Ours, Phi-4-mini-instruct, B1 & 85.60 & 65.17 & 39.25 & 84.33 & 63.60 & 98.00 & 72.66 & 55.80 & 0.111 \\
\bottomrule\end{tabular}}
\end{table}

\begin{table}[!htbp]\centering\small
\caption{Frozen Phi-4-mini-instruct with abstention layers (one run; 9 shifted environments, 9 tests per $\alpha$). Columns as in \cref{tab:abst_small}.}\label{tab:phi}
\resizebox{\textwidth}{!}{%
\begin{tabular}{lrrrrrrr}\toprule
Risk score / threshold rule & Shift AURC $\downarrow$ & Cov@.05 (\%) & Cov@.10 (\%) & Risk@.10 (\%) & Violations $\alpha$=.05/.10/.20 & Viol.@.10 per seed & Gold removed: abstain (\%) \\\midrule
MSP score, pooled LTT $\tau$ & 0.1119 & 31.8 & 49.0 & 6.3 & 1 / 0 / 3 & 0 & 30.9 \\
Split-conformal LAC & 0.1119 & 43.2 & 68.4 & 14.3 & 4 / 4 / 1 & 4 & 30.9 \\
Uncertainty risk, pooled $\tau$ & 0.1291 & 32.9 & 47.9 & 5.2 & 1 / 2 / 1 & 2 & 25.0 \\
Uncertainty risk, ID-only $\tau$ & 0.1291 & 0.0 & 49.0 & 5.6 & 0 / 2 / 1 & 2 & 19.6 \\
Drift inside risk & 0.1609 & 33.5 & 45.7 & 12.1 & 5 / 3 / 4 & 3 & 22.1 \\
\sj{} layer (B5 rule) & 0.1595 & 0.0 & 48.7 & 7.9 & 0 / 2 / 1 & 2 & 14.2 \\
\bottomrule\end{tabular}}
\end{table}

\begin{table}[H]\centering\small
\caption{Computational cost on Kaggle Tesla T4 GPUs.}\label{tab:eff}
\begin{tabular}{lr}\toprule
Quantity & Value \\\midrule
Scoring latency, mean over environments, 0.6B / 4B / Phi-4-mini & 72 / 92 / 89 ms/item \\
Peak GPU memory, scoring 0.6B / 4B & 2.87 / 8.37 GB \\
Peak GPU memory, LoRA training 0.6B / 4B & 3.23 / 9.09 GB \\
LoRA parameters ($r$ = 16, all linear), 0.6B / 4B & 10.1 M / 33.0 M \\
Adapter storage (fp32), 0.6B / 4B & 40.4 / 132.1 MB \\
B2 training + reference cache, 0.6B (2,000 examples) & 8.8 + 2.3 min \\
B2 training + reference cache, 4B (1,000 examples) & 5.3 + 1.4 min \\
Four-stage continual run (training), \sj{} / replay, 0.6B & 12.4 / 13.4 min \\
Four-stage continual run (training), \sj{} / replay, 4B & 17.9 / 18.6 min \\
Session wall-clock, parts A / B / C (two T4 lanes each) & 8.1 / 9.4 / 11.1 h \\
\bottomrule\end{tabular}
\end{table}

\begin{table}[H]\centering\small
\caption{Hyperparameters.}\label{tab:hparams}
\begin{tabular}{ll}\toprule
Parameter & Value \\\midrule
LoRA rank / $\alpha$ / dropout / targets & 16 / 32 / 0 / all linear layers \\
Optimizer, learning rate, weight decay & AdamW, $10^{-4}$, 0.01 \\
Schedule & 10\% warm-up, linear decay to 10\%, 1 epoch, batch 8 \\
Gradient clipping, max length & 1.0, 1024 tokens \\
Anchor weight $\lambda$, truncation & 1, top-64 + tail \\
Risk levels $\alpha$, LTT $\delta$ & $\{0.05,0.10,0.20\}$, 0.1 \\
PCA dim., $k$ ($k$-NN) & 64, 10 \\
Stream lengths (pre / post), severities & 200 / 200, $\{0.25,0.5,1.0\}$ \\
Shift tests (logged; severity) & $p<0.01$, effect $\ge$5 points \\
Gain gate: level $\gamma$, retention margin, held-out items & 0.10, 5 points, 80 \\
$\kappa$, $\lambda_0$, $\lambda_{\mathrm{ret}}$, $\lambda_{\mathrm{dec}}$ & 1.5, 1, 1, 1 \\
Replay per regime, replay ratio, EWC weight, LwF weight & 60, 0.3, 50, 1 \\
Continual eval items per regime, 0.6B / 4B & 400 / 300 \\
Seeds & 20260922, 13, 42 \\
\bottomrule\end{tabular}
\end{table}

\end{document}